\documentclass{article} 
\usepackage{nips15submit_09,times}
\usepackage{url}

\RequirePackage[numbers]{natbib}
\RequirePackage[colorlinks,citecolor=blue,urlcolor=blue]{hyperref}

\title{Regret Guarantees for Item-Item Collaborative Filtering}

\author{
Guy Bresler \hspace{5mm} Devavrat Shah \hspace{5mm} Luis F. Voloch \\
Laboratory for Information and Decision Systems\\
Department of Electrical Engineering and Computer Science\\
Massachusetts Institute of Technology \\
\texttt{\{guy,devavrat,voloch\}@mit.edu} \\
}

\usepackage{amsthm}
\usepackage{amsmath}
\usepackage{amssymb}
\usepackage{mathtools}
\usepackage{bm}
\usepackage{verbatim}
\usepackage{algorithm}
\usepackage{algorithmic}
\usepackage{subcaption}
\usepackage{dsfont}
\usepackage{mathtools}
\usepackage{url}

\usepackage{thm-restate}

\usepackage{cleveref}
\crefname{thm}{Theorem}{Theorems}
\crefname{lmm}{Lemma}{Lemmas}
\crefname{prop}{Proposition}{Propositions}
\crefname{clm}{Claim}{Claims}
\crefname{defin}{Definition}{Definitions}
\crefname{section}{Section}{Sections}
\crefname{pic}{Figure}{Figures}
\crefname{codebox}{algorithm}{Algorithms}
\crefname{eqn}{Equation}{Equations}
\crefname{app}{Appendix}{Appendices}
\crefname{corol}{Corollary}{Corollaries}
\crefname{examp}{Example}{Examples}

\theoremstyle{plain}
\newtheorem{theorem}{Theorem}[section]

\theoremstyle{definition}
\newtheorem{definition}[theorem]{Definition}
\newtheorem{remark}[theorem]{Remark}

\newcommand{\CC}{\mathcal{C}}
\newcommand{\MM}{\mathcal{M}}
\newcommand{\EE}{\mathbb{E}}

\newcommand{\BB}{\mathcal{B}}
\newcommand{\Prob}{\mathbb{P}}
\newcommand{\RR}{\mathcal{R}}

\newcommand{\II}{\mathcal{I}}

\usepackage{clrscode}
\renewcommand{\Then}{%
	\textbf{then}\stepcounter{indent} }
\renewcommand{\Else}{%
	\kill\addtocounter{indent}{-1}%
	\liprint\textbf{else}\>\>\stepcounter{indent}}

\begin{document}

\maketitle

\begin{abstract}
	There is much empirical evidence that item-item collaborative filtering works well in practice. Motivated to understand this, we provide a framework to design and analyze various recommendation algorithms. The setup amounts to online binary matrix completion, where at each time a random user requests a recommendation and the algorithm chooses an entry to reveal in the user's row. The goal is to minimize regret, or equivalently to maximize the number of $+1$ entries revealed at any time. We analyze an item-item collaborative filtering algorithm that can achieve fundamentally better performance compared to user-user collaborative filtering. The algorithm achieves good ``cold-start" performance (appropriately defined) by quickly making good recommendations to new users about whom there is little information. 
\end{abstract}

\section{Introduction}

A natural approach to automated recommendation systems is to use content specific data: similar words in two books' titles suggest that they are similar, and similarly for user's age and geographic location. Recommending based on such content-specific data is called content filtering. In contrast, a technique called collaborative filtering (CF) provides recommendation in a content-agnostic way by exploiting patterns in general purchase or usage data to determine similarity (the term collaborative filtering was coined in~\cite{goldberg92}). For example, if $90 \%$ of users agree on two items, a CF algorithm might recommend the second item to a user who likes first item.

Virtually all industrial recommendation systems use CF. There are two main paradigms in neighborhood-based CF: user-user and item-item. In the user-user paradigm, recommendation to user $u$ is done by finding users similar to $u$ and recommending items liked by these users. In the item-item paradigm, in contrast, items similar to those liked by the user are found and then recommended. Empirical evidence shows that the item-item paradigm performs well (cf. \cite{Linden03} and \cite{korenHandbook}). In this paper we provide theoretical justification by introducing a model for recommendation and analyze the performance of a simple item-item CF algorithm.

\subsection{Model}

We consider a system with $N$ users and collection of items $\II$. For each item $i\in \II$, user $u$ has binary preference $L_{u,i}$ equal to $+1$ (like) or $-1$ (dislike). 
Recommendation systems typically operate in an online setting, meaning that when a user logs into a virtual store (such as Amazon), a recommendation must be made immediately. At each discrete time step $t=1,2,3,\dots$ a uniformly random user $U_t \in \{1,..,N\}$ requests a recommendation. The recommendation algorithm selects an item $I_t$ to recommend from the set of available items $\mathcal{I}$, after which $U_t$ gives feedback $L_{U_t,I_t}$.
The recommendation must depend only on previous feedback: $I_t$ is required to be measurable with respect to the sigma-field generated by the history $(U_1,I_1,L_{U_1,I_1}),\dots,(U_{t-1},I_{t-1},L_{U_{t-1},I_{t-1}})$.

We impose the constraint that the recommendation algorithm may only recommend each item to a given user at most once. This captures the situation where users do not want to watch a movie or read a book more than once and focuses attention on the ability to recommend new items to users.   
For each recommendation, the algorithm may therefore either recommend an item that has been previously recommended to \emph{other} users (in which case it has some information about the item) or recommend a new item from $\mathcal{I}$.  


We are interested in the situation where there are many items, and will assume that $\mathcal{I}$ is infinite. 
For a given item $i$, the corresponding $i^\text{th}$ column $L_{\cdot,i}\in \{-1,+1\}^N$ of each user's preferences is called the \emph{type} of item $i$. 
It is convenient to represent the population of items by a probability measure $\mu$ over $\{-1,+1\}^N$. When the algorithm selects an item that has not yet been recommended, the item's type is drawn from this distribution in an i.i.d. manner. Recommending a new item corresponds to adding a column to the rating matrix, with binary preferences jointly distributed according to $\mu$.   


\subsection{Performance measure}
As is standard in the online decision-making literature, algorithm performance is measured by regret relative to an all-knowing algorithm that makes no bad recommendations. The regret at time $T$ is
$$
\mathcal{R}(T) \triangleq \frac{1}{N} \sum_{t=1}^{T \cdot N} \frac{1}{2} \left( 1- L_{U_t,I_t} \right).
$$
Recall that at time $t$ user $U_t\in\{1,\dots,N\}$ desires a recommendation, $I_t$ is the recommended item, and $L_{u,i}$ is equal to $+1$ (resp. $-1$) if $u$ likes (resp. dislikes) item $i$.  The regret $\mathcal{R}(T)$ is the number of bad recommendations per user after having made an average of $T$ recommendations per user. Dependence on the algorithm is implicit through $I_t$.

\subsection{Main results}
We now describe two high-level objectives in designing a recommendation system and the corresponding guarantees obtained for our proposed algorithm $\proc{item-item-cf}$, which is described in Sections~\ref{sec:algorithm} and~\ref{sec:Explore}. The results are stated in more detail in Section~\ref{sec:mainResult}.

\subsubsection{Cold-start time} With no prior information, the algorithm should give reliable recommendations as quickly as possible. The \emph{cold-start} time of a recommendation algorithm is defined as
	\begin{equation}
	T_{\mathrm{cold-start}} 
	= 
	\min 
	\bigg\{ 
	T + \Gamma \text{ s.t. } ,T,\Gamma \geq 0,
	\EE 
	\left[ \mathcal{R}(T+\Delta) - \mathcal{R}(T) 
	\right] 
	\leq 0.1 \Delta, 
	\forall \Delta >\Gamma
	\bigg\}.
	\end{equation} 
	This is the first time after which the slope of the expected regret is bounded by $0.1$: after $T_{\mathrm{cold-start}}$ the algorithm makes a bad recommendation to a randomly chosen user with probability at most $0.1$.\footnote{The choice $0.1$ is arbitrary. We will assume that users like only a small fraction of the items, so the cold-start time is the minimum time after which the algorithm can recommend significantly better than random.}
	
	\textbf{Our results:} As described in Section~\ref{sec:model}, we assume that each user likes at least $\nu>0$ fraction of the items. In \cref{corol:coldStart} we show that algorithm $\proc{item-item-cf}$ achieves $T_{\mathrm{cold-start}}= \mathcal{\widetilde O} (\frac{1}{\nu})$\footnote{Here and throughout, $r=\mathcal{\widetilde O}(x)$ means $r\leq Cx\log^c x$ for some numerical constants $c,C$. } for $N\geq N_0(d)$. Note that one must typically randomly sample ${\Omega }(\frac1\nu)$ items to find a \emph{single} liked item. Our results show that this amount of time-investment suffices in order to give consistently good recommendations.
	
	\subsubsection{Improving accuracy} The algorithm should give increasingly more reliable recommendations as it gains information about the users and items. This is captured by having \emph{sublinear expected regret} $\EE \left[ \mathcal{R} \left(T \right) \right]=o \left(T \right)$.

\textbf{Our results:} \cref{thm:lowerBound} shows that without assumptions on the item space, it is impossible for any online algorithm to achieve sublinear regret for any positive length of time. In this paper we assume that the item space has doubling dimension (a measure of complexity of the space, defined and motivated later) bounded by $d$. 
In \cref{thm:ManyEpochsRegret} we show that after time $T_\mathrm{cold-start}$ (until which we incur linear regret), algorithm \proc{item-item-cf} achieves sublinear expected regret $\mathcal{\widetilde O} (T^\frac{d+1}{d+2})$ up until a certain time $T_{\mathrm{max}}$. After $T_{\mathrm{max}}$ the expected regret again grows linearly (but with much smaller slope), and this behavior is shown in \cref{thm:lowerBoundLinearRegime} to be unavoidable.
As will be made explicit, performance improves with increasing number of users: $T_{\mathrm{max}}$ (and hence the length of the sublinear time-period) increases with $N$ and the eventual linear slope decreases with $N$, both of which illustrate the so-called collaborative gain. 

\begin{remark}
The mathematical formulation of cold-start time is to the best of our knowledge new. The strong guarantee we obtain on cold-start time (independent of doubling dimension $d$) is distinct from and does not follow as an implication of the sublinear regret result (which \emph{does} depend on $d$).
\end{remark}

\subsection{Related Work}

A latent source model of user types is used by \cite{bresler14} to give performance guarantees for user-user collaborative filtering. The assumptions on users and items are closely related since $K$ items types induce at most $2^K$ user types and vice versa (the $K$ item types liked by a user fully identify the user's preferences, and there are at most $2^K$ such choices). Since we study algorithms that cluster similar items together, in this paper we assume a latent structure of items. We note that unlike the standard mixture model with minimum separation between mixture components (as assumed in \cite{bresler14}), our setup does not have any such gap condition. In contrast, we allow an effectively arbitrary model, and we prove performance guarantees based on a notion of dimensionality of the item space.

Our model can be thought of as a certain extended multi-armed bandit problem. The papers \cite{KleinbergSU13,bubeck2011x} use notions of dimensionality similar to the one in this paper. They assume that the set of arms $\mathcal{X}$ has some geometry: \cite{KleinbergSU13} assumes that the arm space is endowed with a metric, and \cite{bubeck2011x} assumes that the arms have a dissimilarity function (which is not necessarily a metric). The expected rewards are then related to this geometry of the arms. In the former work, the difference in expected rewards is Lipschitz\footnote{That is, $|\EE[r_i]-\EE[r_j]| \leq C \cdot d(i,j)$, where $d$ is the distance between arms, and $C$ is a constant.} in the distance, and in the latter work the dissimilarity function constrains the slope of the reward around its maxima. 

The regret of the algorithms in \cite{KleinbergSU13,bubeck2011x} is $\mathcal{\widetilde O} \big( T^{\frac{d'+1}{d'+2}} \big)$, where $d'$ is a weaker notion (than the one we use) of the covering number of $\mathcal{X}$, and is closely related to the doubling dimension (which we define later) in the case of a metric. The regret bound in \cref{thm:ManyEpochsRegret} for the sublinear regime is of the same form, but two important aspects of our model require a different algorithm and more intricate arguments: $(i)$ in our case no repeat recommendations (i.e. pulling the same arm) can be made to the same user, and $(ii)$ we do not have an oracle for distances between users and items, and instead we must estimate distances by making carefully chosen exploratory recommendations.

Aside from these differences, the nature of the collaborative filtering problem leads to additional novelty relative to existing work on multi-armed bandits. First, we formalize the cold-start problem and prove strong guarantees in this regard. Second, all of our bounds are in terms of system parameters. This allows, for example, to see the role of the number of users $N$ as an important resource allowing for collaboration.

The works of \cite{hazan2012near} and \cite{cesa2013efficient} on online learning and matrix completion are also relevant. In their case, however, the matrix entries to be predicted are not chosen by the algorithm and hence there is no explore-exploit trade-off.  The paper \cite{kleinberg04} considers collaborative filtering under a mixture model in the \emph{offline} setting, and they make separation assumptions between the item types (called genres in their paper). The work \cite{Dabeer13} considers a setting similar to ours (but with finite number of user and item types) and proves certain guarantees on a moving horizon approximation rather than the cumulative anytime regret. The paper \cite{biau2010} proves asymptotic consistency guarantees on estimating the ratings of unrecommended items. The recent paper \cite{Ohannessian15} considers a different model in which repeat recommendations are also not allowed, but they make recommendations by exploiting existing information about users' interests.

It is possible that using the similarities between users, and not just between items as we do, is also useful. This has been studied theoretically in the user-user collaborative filtering framework in \cite{bresler14}, via bandits in a wide variety of settings (for instance \cite{alon14,Slivkins14,cesa2013gang}), with focus on benefits to the cold-start problem \cite{gentile2014online,caron2013mixing}, and in practice (cf. \cite{das2007google,bellogin12}). In this paper, in order to capture the power of purely item-item collaborative filtering, we intentionally avoid using any user-user similarities.

\section{Structure in Data}
\label{sec:model}

The main intuition behind all variants of collaborative filtering is that users and items can typically be clustered in a meaningful way even ignoring context specific data. Items, for example, can often be grouped into a few different types that tend to be liked by the same users. 



\subsection{Need for Structure}
\label{sec:needForStructure}
As discussed, a good recommendation algorithm suggests items to users that are liked but have not been recommended to them before. 
In order to motivate the need for assumptions on the item space, we begin by stating the intuitive result that in the worst case when $\mu$ has little structure, no online algorithm can do better than recommending random items.


\begin{restatable}[Lower Bound]{prop}{thmLowerBound}
	\label{thm:lowerBound}
	Let $\mu$ be the uniform distribution over $\{-1, +1\}^N$. Then for all $T \geq 1$, the expected regret  of any online recommendation algorithm is lower bounded as  $\EE \left[ \mathcal{R}(T) \right] \geq T/2$.  Conversely, recommending a random item at each time step achieves $\EE [\mathcal{R}(T) ]= T/2$.
\end{restatable}

\cref{thm:lowerBound} states that no online algorithm can have sublinear regret for any period of time unless some structural assumptions are made. Hence, to have any \emph{collaborative gain} we need to capture the fact that items tend to come in clusters of similar items. We make two assumptions.
\begin{itemize}
\item[$(\mathbf{A1})$] The distribution $\mu$ over the item space has doubling dimension at most $d$ for a given $d \geq 0$.
\item[$(\mathbf{A2})$] Each user likes a random item drawn from $\mu$ with probability between $\nu$ and $2\nu$, and each item is liked by a fraction  between $\nu$ and $2\nu$ of the users, for a given $\nu \in (0,1/4)$.
\end{itemize}

Assumption $\mathbf{A1}$ captures structure in the item space through the notion of doubling dimension, defined and motivated in \cref{sec:ddDefin}.
Assumption $\mathbf{A2}$ is made to avoid the extreme situations where almost no items are liked (in which case good recommendations are impossible) or most items are liked (in which case the regret benchmark becomes meaningless).


\subsection{Item Types and Doubling Dimension}
\label{sec:itemTypes}
\label{sec:ddDefin}

We endow the $N$-dimensional Hamming cube $\{-1,+1\}^N$ with the following normalized $\ell_1$ metric: for any two item types $x, y \in \{-1,+1\}^N$, define their distance 
$$ \gamma_{x,y} \equiv \gamma (x, y) \stackrel{\triangle}{=} \frac{1}{N} \sum_{k=1}^N \frac{1}{2} | x_k - y_k |. $$
When we write $\gamma_{ij}$ for items $i,j$, we mean the distance between their types, which is the fraction of users that disagree on them.

Let $\BB(x,r) = \{y \in \{-1,+1\}^N: \gamma_{x, y} \leq r\}$ be the ball of radius $r$ centered at $x$.
\begin{definition}[Doubling Dimension] The doubling dimension $d$ of a measure $\mu$ on $\{-1,+1\}^N$ is the least $d$ such that for each $x \in \{-1,+1\}^N$ with $\mu(x)>0$ we have $\sup_{r>0} \frac{\mu(\BB(x, 2 r))}{\mu(\BB(x, r))} \leq 2^d$. A measure with finite doubling dimension is called a doubling measure.
\end{definition}

Measures of low doubling dimension capture the observed clustering phenomenon\footnote{The above definition is a natural adaptation to probability measures on metric spaces of the notion of doubling dimension for metric spaces (cf. \cite{heinonen2001lectures}, \cite{har2006fast} and \cite{dasgupta2014randomized}). As noted in, for instance \cite{dasgupta2014randomized}, this is equivalent to enforcing that $\mu(B_\gamma(x, \alpha r)) \leq \alpha^d \cdot \mu(B_\gamma(x, r))$ for any $r>0$ and any $x \in \{-1,+1\}^N$ with $\mu(x) >0$. For Euclidean spaces, the doubling dimension coincides with the ambient dimension, which reinforces the intuition that metric spaces of low doubling dimension have properties of low dimensional Euclidean spaces.}. It follows directly from the definition, for instance, that a small doubling dimension ensures that the balls around any item type must have a significant mass. In particular, any item type $x \in \{-1,1\}^N$ with $\mu(x) > 0$ has $\mu \big(\BB\left( x,r \right) \big) \geq r^d.$ Appendix~\ref{sec:DDAppendix} contains some examples of measures with low doubling dimension, and the reader is directed there to gain more intuition for the concept. Appendix \ref{sec:DDAppendix} also contains experiments indicating that the doubling dimension is often small in practice, and describes in more detail why it is an assumption strictly more general than the ones made in \cite{bresler14}.




\section{Item-item collaborative filtering algorithm}
\label{sec:algorithm}
In this section we describe our algorithm \proc {item-item-cf}. The algorithm carries out a certain procedure over increasingly longer epochs (blocks of time), where the epoch index is denoted by $\tau \geq 1$. In each epoch the algorithm carefully balances \emph{Explore} and \emph{Exploit} steps.  

In the Explore steps of epoch $\tau$, a partition $\{P_k^{(\tau+1)}\}$ of a set of items is created for use in the subsequent epoch. Each epoch has a target precision $ \varepsilon_\tau$ (specified below) such that if two items $i$ and $j$ are in the same block $P_k^{(\tau+1)}$, then usually $\gamma_{ij} \leq \varepsilon_{\tau+1}$.

In the Exploit steps of epoch $\tau$, the partition $\big\{P_k^{(\tau)} \big\}$ created in the previous epoch is used for recommendation.  Exploit recommendations to a user $u$ are made as follows: $u$ samples a random item $i$ from a random block $P_k^{(\tau)}$, and if $u$ likes $i$ ($L_{u,i}=+1$) then the rest of $P_k^{(\tau)}$ is recommended to $u$ in subsequent Exploit steps. After all items in $P_k^{(\tau)}$ have been recommended to $u$, the user repeats the process by sampling random items in random blocks until liking some item $j$ in $P_{k'}^{(\tau)}$, upon which the rest of $P_{k'}^{(\tau)}$ is recommended.

In the first epoch there is no possibility of exploiting a partition created in a previous epoch, so the algorithm begins with a purely exploratory ``cold-start" period. The pseudo-code of the algorithm is as follows.

\begin{codebox}
	\label{alg:pseudoCode}
	\Procname{$\proc{item-item-cf}(N)$}
	\li \textbf{Algorithm parameters:} 
	\zi  $ \varepsilon_N = \left( \frac{2^{5d+18}}{\nu} \cdot 630 (2d+11)(d+2)^4 \frac{1}{N} \right)^{\frac{1}{d+5}}$, $C=\frac{\nu}{148} \frac{1}{20}$
	\zi $\varepsilon_{\tau} = \max \big( \frac{1}{2^{\tau}}, \varepsilon_N \big) \cdot C$, for $\tau \geq 1$ (target accuracy for epoch)
	\zi $M_\tau = \frac{2^{\max(3.5d,8)}}{\nu} \frac{(3d+1)}{\varepsilon_\tau^{d+2}} \ln (\frac{2}{\varepsilon_\tau})$,  for $\tau \geq 1$ (number of items introduced in epoch)
	\zi $D_\tau = \frac{\nu}{2} M_\tau$, for $\tau \geq 1$ (duration of epoch)
	\li \textbf{Cold-Start:} 
	\zi $\big\{P_k^{(1)} \big\} = \proc{Make-Partition}(M_1,\varepsilon_1,\varepsilon_1)$
	\li \textbf{Subsequent epochs:}
	\zi \For $\tau \geq 1$
	\zi \Do \For $t = 1$ to $N \cdot D_\tau $
	\zi \Do $U_t = $ random user
	\zi w.p. $1-\varepsilon_\tau$: \emph{exploit} $\big\{P_k^{(\tau)} \big\}$ to recommend an item to $U_t$ 
	\zi w.p. $\varepsilon_\tau$: $U_t$ \emph{explores} to help construct partition $\big\{P_k^{(\tau +1)} \big\}$ (see \cref{sec:Explore}) \End \End
\end{codebox}

\section{Explore: making a partition}
\label{sec:Explore}

Recall that during epoch $\tau$ the goal of the explore recommendations is to create a partition $\{ P_k^{(\tau+1)} \}$ of items such that whenever $i,j \in P_k^{(\tau+1)}$ then $\gamma_{ij} \leq \varepsilon_{\tau+1}$. We later prove that this can be done by executing the routine $\proc {make-partition}(M_{\tau+1},\varepsilon_{\tau+1},\varepsilon_{\tau+1})$ described below, which at any point makes recommendations to a randomly chosen user. Hence, given the random user making the recommendation, $\proc{item-item-cf}$ provides explore recommendations in whatever order $\proc{Make-Partition}$ would have recommended (had it been run sequentially)\footnote{For instance, suppose that time $t$ is the first in some epoch $\tau$. We might have that times $t+5$ and $t+30$ are the first two explore recommendations of the epoch, then for those two recommendations the algorithm makes whatever the first two recommendations would have been in $\proc{make-partition}$. If the execution of $\proc{make-partition}$ has finished, the algorithm resorts to an exploit recommendation instead.}. 


\begin{definition}($\varepsilon$-net)
For any $\varepsilon > 0$, a collection of items $\CC$ is called $\varepsilon$-net of the item space represented by
distribution $\mu$ on $\{-1, +1\}^N$ if (a) for any pair $i, j \in \CC$, we have $\gamma_{ij} > \varepsilon/2$, and
(b) for any item of type $\ell$ with $\mu(\ell) > 0$, there exists $i \in \CC$ so that $\gamma_{i\ell} \leq \varepsilon$. 
\end{definition}

$\proc{Make-Partition}$ first finds a \emph{net} $\CC$ for the item space (using the subroutine $\proc{get-net}$ described later). To each item in the net there is associated a block in a partition. $M$ randomly sampled items are assigned to the blocks as follows: for each sampled item $j$, an item $i \in \CC$ is found that is similar to $j$, and $j$ is assigned to the partition block $P_i$ (if there is more than one item $i$ similar to $j$, the algorithm chooses among the relevant blocks at random). Finally, the algorithm breaks up large blocks into blocks of size on the order of $1/\varepsilon$. This guarantees that there will be many blocks in the partition, which turns out to be important in \cref{corol:coldStart} showing brief cold-start time\footnote{It is crucial that blocks in the partition are not too small because we would like the reward for exploration to be large when a user finds a likable item (reward in the sense of many new items to recommend). Although the algorithm does not explicitly ensure that blocks are not too small (as it did in ensuring the blocks are not too large) it comes as a byproduct of a property proven in \cref{clm:DoublingDimensionClaims}, which shows that there are not many items in the net close to any given item $j$.}.

\begin{codebox}
	\Procname{$\proc{Make-Partition}(M,\varepsilon, \delta)$}
	\li $\mathcal{C} = \proc{get-net}(\varepsilon/2,\delta/2)$
	\li $\MM =$ $M$ randomly drawn items from item space
	\li \For each $i \in \mathcal{C}$, let $P_i = \varnothing$ \End
	\li \For each $j \in \MM$, let $S_j = \{i \in \mathcal{C} \mid \proc{similar}(i,j,0.6 \varepsilon, \frac{\delta}{4 M |\CC|})\text{ returns } \const{true}  \}$
	\li \If $|S_j|>0$ \Then 
	 $P_i = P_i \cup \{j\}$, for $i$ chosen u.a.r. from $S_j$ \End \End \End
	\li \For each $i$, \If $|P_i| >  1/\varepsilon$, \Then partition $P_i$ into blocks of size at least $\frac{1}{2 \varepsilon}$ and no more than $1/\varepsilon$
	\li \Return $\{P_i \}$
\end{codebox}

The subroutine $\proc{similar}$ is used in $\proc{make-partition}$; it determines whether most users have the same preference (like or dislike) for two given items $i$ and $j$.This is accomplished by sampling many random users and counting the number of disagreements on the two items.

\begin{codebox}
	\Procname{$\proc{similar}(i,j,\varepsilon,\delta)$}
	\li $q_{\varepsilon,\delta} = 
	\lceil 630 \frac{d+1}{\varepsilon} \ln \left(\frac{1}{\delta}\right) \rceil$
	\li \For $n = 1$ \To $q_{\varepsilon,\delta}$
	\li  \Do sample a uniformly random user $u$
	\li let $X_u = \mathbf{1}_{L_{u,i} \neq L_{u,j}} $ \End
	\li \If $\frac{1}{q_{\varepsilon,\delta}} \sum_{\text{sampled $u$}} X_u \geq 0.9 \varepsilon$ \Then
	\li \Return $\const{False}$ \End
	\li \Return $\const{True}$ \End
\end{codebox}

The subroutine $\proc{Get-net}$ below is a natural greedy procedure for constructing an $\varepsilon$-net. Given parameters $\varepsilon$ and $\delta$, it finds a set of items $\CC$ that is an $\varepsilon$-net for $\mu$ with probability at least $1-\delta$ (proven in the appendix). It does so by keeping a set of items $\CC$ and whenever it samples an item $i$ that currently has no similar item in $\CC$, it adds $i$ to $\CC$. 

\begin{codebox}
	\Procname{$\proc{Get-net}(\varepsilon,\delta)$}
	\li $\mathcal{C} = \varnothing$, $\const{count} = 0$
	\li $\const{max-size} = (4/\varepsilon)^d$, $\const{max-wait} = (\frac{5}{\varepsilon})^d \ln \left(\frac{2 \cdot \const{max-size}}{\delta}\right)$, $\delta' = \delta/\left(4 \cdot \const{max-wait} \cdot \const{max-size}^2 \right) $
	\li \While  $\const{count} \leq \const{max-wait} $ and $| \mathcal{C}| < \const{max-size}$
	\li \Do draw item $i$ from $\mu$
	\li  \If $\proc{similar}(i,j, \varepsilon,\delta')$ for any $j \in \mathcal{C}$ \Then
	\li  $\const{count} = \const{count} +1$ 
	\li 		\Else $\mathcal{C} = \mathcal{C} \cup i$, $\const{count} = 0$ \End \End
	\li \Return $\mathcal{C}$
\end{codebox}

\section{Main Results}
\label{sec:mainResult}

\subsection{Cold-start time and regret bound}
In \cref{sec:algorithm} we described how $\proc{item-item-CF}$ starts recommending items to a user as soon as it finds one item that the user likes. This leads to a short cold-start time. 

\begin{restatable}[Cold-Start Performance]{thm}{corolColdStart}
	\label{corol:coldStart}
	Suppose assumptions $\mathbf{A1}$ and $\mathbf{A2}$ are satisfied. Then the algorithm $\proc{item-item-cf}$ has cold-start time $T_{\mathrm{cold-start}} = \frac{f(\nu,d)}{N}+ \mathcal{\widetilde O}(1/\nu)$.
\end{restatable}

Hence, the algorithm $\proc{item-item-cf}$ has cold-start time $\mathcal{\widetilde O}(1/\nu)$ for $N$ sufficiently large. This differs from that of \cite{bresler14} for user-user paradigm, where the cold-start time increases with user space complexity and the effect is not counteracted with more users present.

The next result shows that after the cold-start period and until a time $T_{\mathrm{max}}$, the expected regret is sublinear.

\begin{restatable}[Regret Upper Bound of $\proc{item-item-cf}$]{thm}{thmManyEpochsRegret}
	\label{thm:ManyEpochsRegret}
	Suppose assumptions $\mathbf{A1}$ and $\mathbf{A2}$ are satisfied. Then $\proc{item-item-cf}$ achieves expected regret 
	\begin{align}\label{eq:regret}
	\EE \left[ \mathcal{R}(T) \right]
	& \leq
	\begin{cases}
	T_{\mathrm{min}} +\alpha(\nu,d) \cdot (T-T_{\mathrm{min}})^{\frac{d+1}{d+2}} \log_2(T-T_{\mathrm{min}})  & T_{\mathrm{min}} < T \leq T_{\mathrm{max}} \\
	\beta+ \varepsilon_{N}  \left(T-T_{\mathrm{max}} \right)  & T > T_{\mathrm{max}}
	\end{cases},
	\end{align}
	where 
	$T_{\mathrm{min}} = \mathcal{\widetilde{O}} \big( \frac{1}{\nu}\big)+ \frac{f(d,\nu)}{N}$, $T_{\mathrm{max}} = g(\nu,d) N^{\frac{d+2}{d+5}}$, $\varepsilon_{N, d, \nu} = h(d,\nu) \left( \frac{1}{N} \right)^{\frac{1}{d+5}} $, 
	$\beta= T_{\mathrm{min}} + \alpha(\nu,d) \cdot (T_{\mathrm{max}}-T_{\mathrm{min}})^{\frac{d+1}{d+2}} \log_2(T_{\mathrm{max}}-T_{\mathrm{min}})$.
\end{restatable}

The reader is directed to the proof in the appendix for the exact constants. Also note that $T_{\mathrm{max}}$ increases with $N$ and the asymptotic slope $\varepsilon_N$ decreases as a function with $N$, both of which illustrate the so-called collaboration gain. Finally, the regret bound in \cref{thm:ManyEpochsRegret} has an asymptotic linear regime. The next result shows that with a finite number of users such linear regret is unavoidable.



\begin{restatable}[Asymptotic linear regret is unavoidable]{thm}{thmLowerBoundLinearRegime}
	\label{thm:lowerBoundLinearRegime}
	Consider an item space $\mu$ satisfying assumptions $\mathbf{A1}$ and $\mathbf{A2}$. Then any online algorithm must have expected asymptotic regret $\EE \left[\mathcal{R} (T) \right] \geq C(\nu,N) \cdot T$, where $C(\nu,N)=(1-2\nu)/N$.
\end{restatable}

\newcommand{\gt}{\widetilde{\gamma}}

\subsection{Comparison with user-user CF}
In this section we will contrast the cold-start performance of user-user collaborative filtering to that of our item-item algorithm. In particular, we give a heuristic argument showing that the cold-start time grows with the complexity of the user space. This is in contrast to our Theorem~\ref{corol:coldStart}, where for any doubling dimension of the item space, if there are sufficiently many users then the cold-start time is independent of system complexity.


We consider a simple scenario with $K$ user clusters. First, let $\gt_{uv}$ denote the probability that users $u$ and $v$ agree on an item randomly drawn from the item space. We have $K$ equally sized clusters of users, such that $\gt_{uv} = 0$ for users $u,v$ in the same cluster, and $\gt_{uv} \in (0.1 \nu, 0.2 \nu)$ for users $u,v$ in different clusters. 


Consider now a given user $u$. A user-user algorithm seeks to find another user $v$ who is similar to $u$, so that the items liked by $v$ can be recommended to $u$. In order to recommend with at most (say) $0.1$ probability of error, the similar user $v$ should have distance $\gt_{uv}$ at most $0.1\nu$. The extra factor $\nu$ is present because inference can only effectively be made from the $\nu$ fraction of liked items.

Concretely, we sample a random user $v$, and attempt to decide if it is from the same cluster as $u$. Suppose $u$ and $v$ have rated $q$ items in common. The problem then reduces to a classical hypothesis test: after observing $q$ items in common from two users, determine whether or not they are from the same cluster. 
The goal is to understand what is the minimal value of $q$ needed so that the above procedure works with at least probability $1/2$.

%
  
We consider the maximum a posteriori rule for deciding that $v$ is from $u$'s cluster. If $u$ and $v$ disagree on any single item, then they cannot be from the same cluster. Conversely, if $u$ and $v$ agree on all $q$ sampled items, the MAP rule declares $v$ to be  from $u$'s cluster only if
$$
\frac{K-1}K(1-0.2\nu)^q \leq \frac1K\,.
$$
This means that if $q$ is too small, we will never declare $v$ to be from $u$'s cluster and therefore will be unable to make recommendations. 
%
%
%
%
%
%
Rearranging gives
$q \geq \Omega \left( \log(K)/\nu \right)$. 

Hence, an algorithm based on user similarity needs at least $T = \Omega \left( \log(K)/\nu \right)$ steps simply to determine if two users are similar to each other, a prerequisite to making good recommendations. In contrast, we have shown that \proc {item-item-cf} achieves cold start time $\mathcal{\widetilde O}(1/\nu)$, which in particular does not increase with the complexity of the item space.

This contrast between cold-start times highlights the \emph{asymmetry between item-item and user-user} collaborative filtering. It is much faster to compare two items than it is to compare two users: it takes a long time to make many recommendations to two particular users, but comparing two items can be done in \emph{parallel} by sampling different users.

\section{Discussion}


This paper analyzes a collaborative filtering algorithm based on item similarity, and proves guarantees on its regret. 
Our algorithm exploits structure only in the item space. It would be desirable to have a matching lower bound, in the spirit of lower bound for multi-armed bandits in metric spaces shown in \cite{KleinbergSU13} and \cite{bubeck2011x}. Furthermore, many practitioners use a hybrid of user-user and item-item paradigms \cite{wang2006unifying} and \cite{verstrepen2014unifying}, and formally analyzing  such algorithms is an open problem.

Finally, the main challenge of the cold-start problem is that initially we do not have any information about item-item similarities. In practice, however, some similarity can be inferred via content specific information. For instance, two books with similar words in the title can have a prior for having a higher similarity than books with no similar words in the title. In practice such hybrid content/collaborative filtering algorithms have had good performance \cite{melville2002content}. Formally analyzing such hybrid algorithms has not been done and can shed light onto how to best combine content information with the collaborative filtering information.

\newpage

{\small
	\bibliographystyle{unsrt}
	\bibliography{database}

\begin{thebibliography}{10}

\bibitem{goldberg92}
David Goldberg, David Nichols, Brian~M. Oki, and Douglas Terry.
\newblock Using collaborative filtering to weave an information tapestry.
\newblock {\em Commun. ACM}, 1992.

\bibitem{Linden03}
Greg Linden, Brent Smith, and Jeremy York.
\newblock Amazon.com recommendations: Item-to-item collaborative filtering.
\newblock {\em IEEE Internet Computing}, 7(1):76--80, 2003.

\bibitem{korenHandbook}
Yehuda Koren and Robert Bell.
\newblock Advances in collaborative filtering.
\newblock In {\em Recommender Systems Handbook}, pages 145--186. Springer US,
  2011.

\bibitem{bresler14}
Guy Bresler, George Chen, and Devavrat Shah.
\newblock A latent source model for online collaborative filtering.
\newblock In {\em NIPS}, 2014.

\bibitem{KleinbergSU13}
Robert Kleinberg, Aleksandrs Slivkins, and Eli Upfal.
\newblock Bandits and experts in metric spaces.
\newblock {\em arXiv preprint arXiv:1312.1277}, 2013.

\bibitem{bubeck2011x}
S{\'e}bastien Bubeck, R{\'e}mi Munos, Gilles Stoltz, and Csaba Szepesvari.
\newblock X-armed bandits.
\newblock {\em JMLR}, 2011.

\bibitem{hazan2012near}
Elad Hazan, Satyen Kale, and Shai Shalev-Shwartz.
\newblock Near-optimal algorithms for online matrix prediction.
\newblock {\em COLT}, 2012.

\bibitem{cesa2013efficient}
Nicol{\`o} Cesa-Bianchi and Ohad Shamir.
\newblock Efficient transductive online learning via randomized rounding.
\newblock In {\em Empirical Inference}, pages 177--194. Springer, 2013.

\bibitem{kleinberg04}
Jon Kleinberg and Mark Sandler.
\newblock Using mixture models for collaborative filtering.
\newblock In {\em STOC}, 2004.

\bibitem{Dabeer13}
Onkar Dabeer.
\newblock Adaptive collaborating filtering: The low noise regime.
\newblock In {\em ISIT}, 2013.

\bibitem{biau2010}
Gerard Biau, Benoit Cadre, and Laurent Rouviere.
\newblock Statistical analysis of k-nearest neighbor collaborative
  recommendation.
\newblock {\em The Annals of Statistics}, 38(3):1568--1592, 2010.

\bibitem{Ohannessian15}
Laurent Massoulie, Mesrob Ohannessian, and Alexandre Proutiere.
\newblock Greedy-bayes approach for targeted news dissemination.
\newblock {\em Sigmetrics}, 2015.

\bibitem{alon14}
Noga Alon, Nicol{\`{o}} Cesa{-}Bianchi, Claudio Gentile, Shie Mannor, Yishay
  Mansour, and Ohad Shamir.
\newblock Nonstochastic multi-armed bandits with graph-structured feedback.
\newblock {\em arXiv Technical Report arXiv:1409.8428}, 2014.

\bibitem{Slivkins14}
Aleksandrs Slivkins.
\newblock Contextual bandits with similarity information.
\newblock {\em The Journal of Machine Learning Research}, 2014.

\bibitem{cesa2013gang}
Nicolo Cesa-Bianchi, Claudio Gentile, and Giovanni Zappella.
\newblock A gang of bandits.
\newblock In {\em NIPS}, 2013.

\bibitem{gentile2014online}
Claudio Gentile, Shuai Li, and Giovanni Zappella.
\newblock Online clustering of bandits.
\newblock {\em ICML}, 2014.

\bibitem{caron2013mixing}
St{\'e}phane Caron and Smriti Bhagat.
\newblock Mixing bandits: A recipe for improved cold-start recommendations in a
  social network.
\newblock In {\em Workshop on Social Network Mining and Analysis}, 2013.

\bibitem{das2007google}
Abhinandan~S Das, Mayur Datar, Ashutosh Garg, and Shyam Rajaram.
\newblock Google news personalization: scalable online collaborative filtering.
\newblock In {\em WWW}, 2007.

\bibitem{bellogin12}
Alejandro Bellogin and Javier Parapar.
\newblock Using graph partitioning techniques for neighbour selection in
  user-based collaborative filtering.
\newblock In {\em RecSys}, 2012.

\bibitem{heinonen2001lectures}
Juha Heinonen.
\newblock {\em Lectures on analysis on metric spaces}.
\newblock Springer, 2001.

\bibitem{har2006fast}
Sariel Har-Peled and Manor Mendel.
\newblock Fast construction of nets in low-dimensional metrics and their
  applications.
\newblock {\em SIAM Journal on Computing}, 35(5):1148--1184, 2006.

\bibitem{dasgupta2014randomized}
Sanjoy Dasgupta and Kaushik Sinha.
\newblock Randomized partition trees for nearest neighbor search.
\newblock {\em Algorithmica}, 2014.

\bibitem{wang2006unifying}
Jun Wang, Arjen De~Vries, and Marcel Reinders.
\newblock Unifying user-based and item-based collaborative filtering approaches
  by similarity fusion.
\newblock In {\em ACM SIGIR}, 2006.

\bibitem{verstrepen2014unifying}
Koen Verstrepen and Bart Goethals.
\newblock Unifying nearest neighbors collaborative filtering.
\newblock In {\em Recsys}, 2014.

\bibitem{melville2002content}
Prem Melville, Raymond~J Mooney, and Ramadass Nagarajan.
\newblock Content-boosted collaborative filtering for improved recommendations.
\newblock In {\em AAAI/IAAI}, 2002.

\bibitem{goldberg2001eigentaste}
Ken Goldberg, Theresa Roeder, Dhruv Gupta, and Chris Perkins.
\newblock Eigentaste: A constant time collaborative filtering algorithm.
\newblock {\em Information Retrieval}, 2001.

\bibitem{riedl1998movielens}
J~Riedl and J~Konstan.
\newblock Movielens dataset, 1998.

\bibitem{mcdiarmid1998concentration}
Colin McDiarmid.
\newblock Concentration.
\newblock In {\em Probabilistic methods for algorithmic discrete mathematics}.
  Springer, 1998.

\end{thebibliography}

}

\newpage
\appendix

\section{Correctness of Explore}
\label{sec:algoCorrectness}

This section of the Appendix establishes correctness of the explore procedure as well as some of its properties that will be utilized for establishing the main result of the paper. Concretely, we will prove that with high probability the procedure $\proc{make-partition}$ produces a partition of similar items during each epoch. To that end, in \cref{sec:similar}, 
we prove that \proc {similar} succeeds in deciding whether two items are close to each other. 
In \cref{sec:makePartition} we prove that the procedure \proc {get-net} succeeds in finding a set of items 
that is an $\varepsilon$-net for $\mu$. We then put all the pieces together and prove that $\proc {make-partition}$, the routine at which the explore recommendations are aimed at completing, succeeds in creating a partition of similar items. Finally, in \cref{sec:suffExploration} we prove that with high probability during any given epoch there will be enough explore recommendations.

\subsection{Guarantees for \proc {similar}}
\label{sec:similar}

The procedure \proc {similar} is used throughout \proc {get-net} and \proc {make-partition}. It tests whether two items are approximately $\varepsilon$-close to each other.

\begin{restatable}{lmm}{lmmPreIsSimilarNew}
	\label{lmm:PreIsSimilarNew}
	Let $i$ and $j$ be arbitrary items, $\delta, \varepsilon \in (0,1)$, and $S_{i,j}$ be the event that $\proc{similar}(i,j,\varepsilon,\delta)$ returns $\const{true}$. Then we have that
	\begin{itemize}
		\item[(i)] if $\gamma_{i,j} \leq 0.8 \varepsilon$, then $\Prob \left( S_{ij} \right) \geq 1-\delta$, and
		\item[(ii)] if $\gamma_{ij} \in [ k \varepsilon, (k+1)\varepsilon )$ where $k \in \{1,...,\lfloor \frac{1}{\varepsilon} \rfloor \}$, then $\Prob \left( S_{ij} \right) \leq \frac{\delta}{4} \left(  \frac{1}{ 4k} \right)^d \frac{1}{k^2}$.
	\end{itemize}
\end{restatable}
\begin{proof}
	Let us begin with case $(i)$, where $\gamma_{ij} \leq 0.8 \varepsilon$. Let $A_n$ be the event that the $n^{th}$ randomly chosen user disagrees on $i$ and $j$ (i.e. that user likes exactly one of $i$ and $j$), and note that $\mathbb{E} \left( \sum_{n=1}^{q_{\varepsilon,\delta}} \mathbf{1}_{A_n} \right) \leq 0.8 \varepsilon q_{\varepsilon,\delta}$. Then, by the Chernoff Bound (stated in \cref{thm:ChernoffBound}), we get
	\begin{equation}
	\Prob
	\left(
	S_{ij}^c \mid \gamma_{ij} \leq 0.8 \varepsilon
	\right)
	=
	\mathbb{P}
	\left(
	\sum_{n=1}^{q_{\varepsilon,\delta}} \mathbf{1}_{A_n}
	\geq
	0.9 \varepsilon q_{\varepsilon,\delta}
	\right)
	\leq
	\mathbb{P}
	\left(
	\sum_{n=1}^{q_{\varepsilon,\delta}} \mathbf{1}_{A_n}
	\geq
	\left(
	1+.1
	\right)
	\mathbb{E}
	\left(
	\sum_{n=1}^{q_{\varepsilon,\delta}} \mathbf{1}_{A_n}
	\right)
	\right)
	\end{equation}
	$$
	\leq
	\exp
	\left(
	-\frac{.1^2}{2+.1}
	\mathbb{E}
	\left(
	\sum_{n=1}^{q_{\varepsilon,\delta}} \mathbf{1}_{A_n} 
	\right)
	\right)
	\leq
	\exp
	\left(
	-\frac{1}{210}
	0.8 \varepsilon q_{\varepsilon,\delta} 
	\right).
	$$
	
	Now, since $q_{\varepsilon,\delta} = 
	\lceil 630 \frac{d+1}{\varepsilon} \ln \left(\frac{1}{\delta}\right) \rceil \geq 210 \cdot \frac{5}{4} \frac{1}{\varepsilon}\ln\left(\frac{1}{\delta}\right)$, we get
	\begin{equation}
	\Prob
	\left(
	S_{ij}^c \mid \gamma_{ij} \leq 0.8 \varepsilon
	\right)
	=
	\mathbb{P}
	\left(
	\sum_{n=1}^{q_{\varepsilon,\delta}} \mathbf{1}_{A_n}
	\geq
	0.9 \varepsilon q_{\varepsilon,\delta}
	\right)
	\leq
	\delta.
	\end{equation}
	
	Let us now consider case $(ii)$, where $\gamma_{ij} \in \left[k \varepsilon, (k+1) \varepsilon \right)$. As before, let $A_n$ be the event that the $n^{th}$ randomly chosen user disagrees on $i$ and $j$. Then, since $k \geq 1$ and again by the Chernoff bound we get
	\begin{equation}
	\mathbb{P}
	\big(
	S_{ij} \mid \gamma_{ij} \in [ k \varepsilon, (k+1)\varepsilon )
	\big)
	=
	\mathbb{P}
	\left(
	\sum_{n=1}^{q_{\varepsilon,\delta}} \mathbf{1}_{A_n}
	\leq
	0.9 \varepsilon q_{\varepsilon,\delta}
	\right)
	\end{equation}
	\begin{equation*}
	\leq
	\mathbb{P}
	\left(
	\sum_{n=1}^{q_{\varepsilon,\delta}} \mathbf{1}_{A_n}
	\leq
	\left(
	1-.1
	\right)
	\mathbb{E}
	\left(
	\sum_{n=1}^{q_{\varepsilon,\delta}} \mathbf{1}_{A_n}
	\right)
	\right)
	\end{equation*}
	\begin{equation*}
	\leq
	\exp
	\left(
	-\frac{.1^2}{2+.1}
	\mathbb{E}
	\left(
	\sum_{n=1}^{q_{\varepsilon,\delta}} \mathbf{1}_{A_n}
	\right)
	\right)
	\leq
	\exp
	\left(
	-\frac{1}{210}
	q_{\varepsilon,\delta}  k\varepsilon
	\right).
	\end{equation*}
	In order to get the desired conditions we need to show that
	\begin{equation}
	\exp
	\left(
	- \frac{1}{210}
	q_{\varepsilon,\delta}  k\varepsilon
	\right)
	\leq
	\frac{\delta}{4} \left(\frac{1}{4k}\right)^d \frac{1}{k^2}.
	\end{equation}
	By taking natural log of both sides we get
	\begin{equation}
	\frac{1}{210}
	q_{\varepsilon,\delta} k\varepsilon
	\geq
	d
	\ln
	\left( 4k
	\right)
	+
	\ln
	\left( \frac{4 k^2}{\delta}
	\right),
	\end{equation}
	which is in turn at most
	$
	(d+1)
	\ln
	\left( \frac{16k^3}{\delta}
	\right).
	$
	Hence, it suffices to show that 
	\begin{equation}
	\frac{1}{210}
	q_{\varepsilon,\delta} k\varepsilon
	\geq
	(d+1)
	\ln
	\left( 
	\frac{16k^3}{\delta}
	\right).
	\end{equation}
	However, since we have $q_{\varepsilon,\delta} \geq 630 \frac{d+1}{ \varepsilon} \ln \left(\frac{1}{\delta}\right)$, we get
	\begin{equation}
	\frac{1}{210} q_{\varepsilon,\delta} k \varepsilon \geq \frac{1}{210} 630 \frac{d+1}{ \varepsilon} \ln \left(\frac{1}{\delta}\right) k\varepsilon = 3 (d+1)  k \ln \left( \frac{1}{\delta} \right)
	.
	\end{equation}
	Now since $3k \geq \ln(16 k^3)$ for $k \geq 1$, we conclude that 
	\begin{equation}
	\mathbb{P}
	\big(
	S_{ij} \mid \gamma_{ij} \in [ k \varepsilon, (k+1)\varepsilon )
	\big)
	=
	\mathbb{P}
	\left(
	\sum_{n=1}^q \mathbf{1}_{A_n}
	\leq
	0.9 \varepsilon q_{\varepsilon,\delta}
	\right)
	\leq
	\frac{\delta}{4} \left(  \frac{1}{ 4k} \right)^d \frac{1}{k^2},
	\end{equation}
	as desired.
\end{proof}

In the Lemma above we showed that given two items $i$ and $j$, the routine \proc {similar} can tell that the items are $\varepsilon$-close when $\gamma_{i,j} \leq 0.8 \varepsilon$, and it can tell that they are not when $\gamma_{i,j} \geq \varepsilon$. Furthermore, the Lemma states that the probability of a false-positive decreases extremely fast as the items get farther apart. The Lemma below shows that, when one of the items is drawn from $\mu$, \proc {similar} still works and that the false positive rate is small, despite the possibility that the it may be much more likely to draw an item that is far from $i$. In \cref{lmm:IsSimilarNew} below we use the doubling dimension of $\mu$ for the first time, and in this context the doubling dimension guarantees that \proc {similar} (which is a random projection) preserves relative distances.

\begin{restatable}{lmm}{lmmIsSimilarNew}
	\label{lmm:IsSimilarNew}
	Let $i$ be an arbitrary item, let $J$ be a randomly drawn item from an item space $\mu$ of doubling dimension $d$, and let $S_{iJ}$ be the event that $\proc{similar}(i,J,\varepsilon,\delta)$ returns $\const{True}$. Then we have $\mathbb{P} \left( \gamma_{iJ} \geq \varepsilon \mid S_{iJ} \right) \leq \delta$.
\end{restatable}

\begin{proof}
	By Bayes' rule we get
	\begin{equation}
	\Prob \left( \gamma_{i,J} \geq \varepsilon \mid S_{iJ} \right)
	=
	\frac{ \Prob \left( \gamma_{i,J} \geq \varepsilon , S_{iJ}  \right) }
	{ \Prob \left( \gamma_{i,J} \geq \varepsilon, S_{iJ} \right)
		+
		\Prob \left( \gamma_{i,J} < \varepsilon , S_{iJ} \right)},
	\end{equation}
	where the probability is with the respect to the random choice of $J$ and the random users in  $\proc{similar}$.
	Now if
	\begin{equation*}
	\label{eqSufficientConditionNew}
	\tag{$\star$}
	\delta
	\mathbb{P} \left(  S_{iJ} , \gamma_{iJ} < \varepsilon \right) 
	\geq 
	\left(
	1-\delta
	\right)
	\mathbb{P} \left(  S_{iJ} , \gamma_{iJ} \geq \varepsilon \right) \end{equation*}
	holds, we get
	\begin{equation}
	\Prob \left( \gamma_{i,J} \geq \varepsilon \mid S_{iJ} \right)
	=
	\frac{ 1}
	{ 1
		+
		\frac{ \Prob \left( \gamma_{i,J} < \varepsilon , S_{iJ} \right)}{\Prob \left( \gamma_{i,J} \geq \varepsilon , S_{iJ} \right) } }
	\leq
	\frac{1}{1+\frac{1-\delta}{\delta}} = \delta.
	\end{equation}
	Hence, it suffices to show (\ref{eqSufficientConditionNew}).
	Recall that $\BB \left( i ,r\right)$ is the ball of radius $r$ centered at $i$, and note that
	\begin{equation}
	\mathbb{P} \left( \gamma_{iJ} < \varepsilon , S_{iJ} \right) 
	\geq
	\mathbb{P} \left( S_{iJ} \mid \gamma_{iJ} \leq \varepsilon/2 \right) 
	\mu \left( \BB \left( i ,\varepsilon/2 \right) \right),
	\end{equation}
	and
	\begin{equation}
	\mathbb{P} \left( \gamma_{iJ} \geq \varepsilon , S_{iJ} \right) 
	=
	\sum_{k=0}^{ \lceil \log_2 \left(\frac{1}{\varepsilon}\right) \rceil }
	\mathbb{P} \left(  S_{iJ} \mid \gamma_{iJ} \in [2^k \varepsilon,2^{k+1} \varepsilon ) \right) 
	\mu \left( \BB \left( i , 2^{k+1} \varepsilon \right) -  \BB \left( i , 2^{k} \varepsilon \right)  \right)
	\end{equation}
	\begin{equation}
	\leq
	\sum_{k=0}^{ \lceil \log_2 \left(\frac{1}{\varepsilon}\right) \rceil }
	\mathbb{P} \left( S_{iJ} \mid \gamma_{iJ} \in [2^k \varepsilon,2^{k+1} \varepsilon ) \right) 
	\mu \left( \BB \left( i , 2^{k+1} \varepsilon \right)  \right).
	\end{equation}
	
	Let us first lower bound $\mathbb{P} \left( S_{iJ} , \gamma_{ij} \leq \varepsilon/2 \right) \mu \left( \BB \left( i ,\varepsilon/2 \right)  \right)$. Let $ p \triangleq \mu \left( \BB \left( i ,\varepsilon/2 \right)  \right)$. Then by \cref{lmm:PreIsSimilarNew} we get that 
	\begin{equation}
	\mathbb{P} \left( S_{iJ} \mid \gamma_{iJ} \leq \varepsilon/2 \right) 
	\mu \left( \BB \left( i ,\varepsilon/2 \right) \right)
	\geq
	\left(1-\delta\right) p.\end{equation}

	We will now upper bound $\mathbb{P} \left( S_{iJ} , \gamma_{iJ} \geq \varepsilon \right) $.
	 Using the doubling dimension of the item space, which implies that $\mu \left( \BB_\gamma \left( i ,2^{k+1} \varepsilon \right)  \right) \leq \left(2^{k+2}\right)^d p $, we also get that
	\begin{equation}
	\mathbb{P} \left( \gamma_{iJ} \geq \varepsilon , S_{iJ} \right) 
	\leq
	\sum_{k=0}^{ \lceil \log_2 \left(\frac{1}{\varepsilon}\right) \rceil }
	\mathbb{P} \left( S_{iJ} \mid \gamma_{iJ} \in [2^k \varepsilon,2^{k+1} \varepsilon )  \right) 
	\mu \left( \BB_\gamma \left( i ,2^{k+1} \varepsilon \right)  \right)
	\end{equation}
	\begin{equation}
	\leq
	\sum_{k=0}^{ \lceil \log_2 \left(\frac{1}{\varepsilon}\right) \rceil }
	\mathbb{P} \left(  S_{iJ} \mid \gamma_{iJ} \in [2^k \varepsilon,2^{k+1} \varepsilon )  \right) 
	\left(2^{k+2}\right)^d p.
	\end{equation}
	We now use the second half of \cref{lmm:PreIsSimilarNew}, and arrive at
	\begin{equation}
	\mathbb{P} \left( \gamma_{iJ} \geq \varepsilon ,S_{iJ} \right) 
	\leq
	\sum_{k=0}^{ \lceil \log_2 \left(\frac{1}{\varepsilon}\right) \rceil }
	\left(\frac{\delta}{4} \left( \frac{1}{4 \cdot 2^{k}} \right)^d \frac{1}{2^{2k}} \right)
	\left(2^{k+2}\right)^d p
	\leq
	p \frac{\delta}{4}  \sum_{k=0}^{ \infty	} \frac{1}{2^{2k}}
	\leq
	p \frac{\delta}{2}.
	\end{equation}
	We can now check that indeed sufficient condition from  \cref{eqSufficientConditionNew} is satisfied:
	\begin{equation}
	\delta
	\mathbb{P} \left( S_{iJ} , \gamma_{iJ} < \varepsilon \right) 
	\geq
	\delta
	p \left(1-\delta\right)
	\geq
	\frac{\delta}{2}
	p \left(1-\delta\right)
	\geq
	\left(1-\delta\right)
	\mathbb{P} \left(S_{iJ}, \gamma_{iJ} \geq \varepsilon \right),
	\end{equation}
	which completes the proof.
\end{proof}

\subsection{Making the Partition}
\label{sec:makePartition}
In the previous section we proved that the procedure $\proc {similar}$ works well in deciding whether two items are similar to each other at some desired precision. In this section, we will prove that with $\proc {similar}$ as a building block can partition items into blocks of similar items.

We will begin by proving that the subroutine $\proc{get-net}$, used in the beginning of $\proc {make-partition}$, succeeds at producing an $\varepsilon$-net of items with high probability.

\begin{restatable}{lmm}{lemmaClusteringWorksNew}
	\label{lmm:ClusteringWorksNew}
	With probability at least $1-\delta$ the routine $\proc{Get-net}(M,\varepsilon,\delta)$ returns an $\varepsilon$-net for $\mu$ that contains at most $\left( \frac{4}{\varepsilon} \right)^d$ items.
\end{restatable}
\begin{proof}
	Let us first settle some notation. Let $\CC_{final}$ be the set returned \proc{get-net}$(\varepsilon,\delta)$, let $\CC_r$ be the set $\CC$ when it had $r$ items, and let $\MM_r$ be the set of random items drawn when $\CC$ had $r$ items. Furthermore, denote by $P$ be event that for each $i,j \in \CC_{final}$ we have $\gamma_{i,j} \geq \varepsilon/2$, and $C$ the event that for each item $i$ there exists a $j \in \CC_{final}$ such that $\gamma_{i,j}  \leq \varepsilon$. Furthermore, let $E_{i,j}$ be the event $\{ S_{i,j}, \gamma_{ij} > 0.5 \varepsilon \} \cup \{ S_{i,j}^c, \gamma_{ij} < 0.8 \varepsilon \}$, and 
	$$
	E 
	\triangleq
	\bigcup_{r=0}^{|\CC_{final} |-1}
	\bigcup_{j \in \MM_r}
	\bigcup_{c \in \CC_r} E_{c,j}.
	$$
	Intuitively, the event $E$ happens when some call to $\proc {similar}$ returned an erroneous answer. We will show that 
	\begin{itemize}
		\item[$(A)$] $\Prob \left(E \right) \leq \delta/2$,
		\item[$(B)$] $\Prob \left(P^c \mid E^c \right) = 0$, and
		\item[$(C)$] $\Prob \left(C^c \mid E^c \right) \leq \delta/2$,
	\end{itemize}
	which together show that \proc{get-net}$(\varepsilon,\delta)$ returns an $\varepsilon$-net with probability at least $1-\delta$.
	%

	\emph{Proof of $(A)$}: By a union bound we get
	$$
	\Prob \left(E \right)
	\leq
	\sum_{r=0}^{|\CC_{final} |-1}
	\sum_{j \in \MM_r}
	\sum_{c \in \CC_r} 
	\Prob \left( E_{c,j} \right)
	\leq
	\sum_{r=0}^{\const{max-size}}
	\sum_{j \in \MM_r}
	\sum_{c \in \CC_r} 
	\Prob \left( E_{c,j} \right),
	$$
	and since there are at most $\const{Max-wait}$ items in $\MM_r$ and at most $\const{max-size}$ items in $\CC_r$ we get
	$$
	\Prob \left(E \right)
	\leq
	\left( \const{max-size} \right)^2 
	\cdot \const{max-wait}
	\cdot \Prob \left( E_{c,j} \right)
	\leq
	\delta/4.
	$$
	where the last inequality follows since \cref{lmm:PreIsSimilarNew} gives us that $$ \Prob \left( E_{c,j} \right) \leq \delta' = \delta/\left(4 \cdot \const{max-wait} \cdot \const{max-size}^2 \right) .$$
	
	\emph{Proof of $(B)$}: Note that if there are two items $i, j \in \CC$ such that $\gamma_{i,j} \leq \varepsilon/2$, then this must have happened as a result of some erroneous response of $\proc {similar}\left(i,j,0.6 \varepsilon, \delta' \right)$. However, since we are conditioning on $E^c$ no such erroneous response can occur.
	
	\emph{Proof of $(C)$}: Let us consider the two cases, when $|\CC_{final}| = \const{max-size}$, and when $|\CC_{final}|<\const{max-size}$. Then
	$$
	\Prob \left( C^c \mid E^c \right)
	\leq
	\Prob \left( C^c \mid E^c, |\CC_{final}| = \const{max-size} \right)
	+
	\Prob \left( C^c \mid E^c, |\CC_{final}| < \const{max-size} \right),
	$$
	and we will show that
	\begin{itemize}
		\item[($C_1$)]$\Prob \left( C^c \mid E^c, |\CC_{final}| = \const{max-size}| \right) = 0$,
		and
		\item[($C_2$)]$\Prob \left( C^c \mid E^c, |\CC_{final}| < \const{max-size}| \right) = \delta/2 $,
	\end{itemize}
	which together prove $(C)$.
	
	\emph{Proof of $(C_1)$}: Note that since we are conditioning on $E_c$, which in turn implies that $\CC_{final}$ is a packing, we have that $\gamma_{ij} \geq \varepsilon/2$ for each $i,j \in \CC_{final}$, we get that $B_\gamma (i, \varepsilon/4) \cap \BB (j, \varepsilon/4) = \emptyset$ for each $i,j \in \CC_{final}$ as well. Hence
	$$
	\mu \left( \bigcup_{i \in \CC_{final}} \BB \left(i, \varepsilon/2 \right) \right)
	\geq
	\mu \left( \bigcup_{i \in \CC_{final}} \BB \left( i, \varepsilon/4 \right) \right)
	=
	\sum_{i \in \CC_{final}} \mu \left(  \BB \left( i, \varepsilon/4 \right) \right).
	$$
	Now by the doubling dimension condition we get that $\mu \left(  \BB \left( i, \varepsilon/2 \right) \right) \geq \left(\frac{\varepsilon}{4} \right)^d$, and hence can conclude that
	$$
	\mu \left( \bigcup_{i \in \CC_{final}} \BB \left(i, \varepsilon/2 \right) \right)
	\geq
	\left(\frac{\varepsilon}{4} \right)^d
	\cdot
	\left(\frac{4}{\varepsilon} \right)^d
	=1,
	$$
	and hence for any item there exists an item in $\CC_{final}$ such that $\gamma_{i,j} \leq \varepsilon/2$.
	
	\emph{Proof of $(C_2)$}: Consider now the case in which $|\CC_{final}| < \const{max-size}$. This means that at some iteration $r \in \{0,...,\const{max-size}-1 \}$ of the while loop of there existed an item $j$ such that $\gamma_{ij} > \varepsilon$ for each $i \in \CC_r$ but the algorithm nevertheless terminated and returned $\CC_r$. Let $T_r$ be the event that the algorithm terminated at round $r$ while there still existed an item $j$ which is not $\varepsilon$ close to any item in $\CC_r$. Then $C^c \subset \bigcup_{r=0}^{|\CC_{final}|} T_r$, and hence
	$$
	\Prob \left( C^c \mid E^c, |\CC_{final}| < \const{max-size} \right)
	\leq
	\sum_{r=0}^{|\CC_{final}|} \Prob \left( T_r \mid E^c \right),
	$$
	and so it suffices for $C_2$ to show that $\Prob \left( T_r \mid E^c \right) \leq \frac{\delta}{2}\frac{1}{\const{max-size}}$ for each $r$.
	
	To show that  $\Prob \left( T_r \mid E^c \right) \leq \frac{\delta}{2}\frac{1}{\const{max-size}}$, we will first observe that if there is an item $j$ which is $\varepsilon$ far from all of $\CC_r$, then the ball $B\left(j, \varepsilon/5 \right)$ must have significant mass which is all also not close to any item in $\CC_r$. We will then conclude, by a standard coupon collector argument, that this mass is found with high probability.
	
	By the doubling dimension condition, the ball $B(j,\varepsilon/5)$ must have mass at least $(\varepsilon/5)^d$. Let $M_r$ be the event that no item in $B(j,\varepsilon/5)$ was sampled during the $r^{th}$ of the while loop. Then
	$$
	\Prob \left( T_r \mid E^c \right) 
	\leq
	\Prob \left( T_r \mid E^c,M_r^c \right)
	+
	\Prob \left( M_r \mid E^c \right).
	$$
	
	Note that given $M_r^c$, which implies that an item $j$ which is at least $0.8 \varepsilon$ away from all of $\CC_r$ was sampled, the event $T_r$ happens only if $j$ is judged to be similar to some $c \in \CC_r$. However, since we're conditioning on $E_c$ that cannot happen, and we get $\Prob \left( T_r \mid E^c,M_r^c \right) = 0$.
	
	Finally, we will use the coupon collector argument and show that $\Prob \left( M_r \mid E^c \right) \leq \delta/(2 \cdot \const{max-size})$. The event $M_r$, that no item in $B(j,\varepsilon/2)$ was sampled during the $r^{th}$ iteration of the loop happens with probability at most
	$$
	\left(1-\left(\frac{\varepsilon}{5}\right)^d \right)^{\const{max-wait}}
	\leq
	\exp \left( - \left(\frac{\varepsilon}{5}\right)^d \const{max-wait} \right) \leq
	\frac{\delta}{2 \cdot \const{max-size}},
	$$
	as we wished.
\end{proof}

It is now only left to prove that the main tool used during exploration, \proc{Make-Partition}, indeed produces a partition of similar items.

We now prove that with high probability the procedure \proc {make-partition} 
creates a partition of similar items. Furthermore, the additional properties stated in the 
Lemma, regarding the size of the blocks, will be crucial later in ensuring a quick cold-start performance.

\begin{restatable}{lmm}{lmmPartitionConditions}
	\label{lmm:partitionConditions}
	Let $\varepsilon, \delta \in (0,1)$, and let $M \geq 12 \cdot \left(\frac{12}{\varepsilon}\right)^{d+1} \ln \left( \frac{2}{ \delta} \left( \frac{8}{\varepsilon} \right)^d \right)$. Then with probability at least $1-\delta$ the subroutine $\proc{make-partition}(M,\varepsilon,\delta)$ returns a
	partition $\{P_k\}$ of a subset of $M$ randomly drawn items such that
	\begin{itemize}
		\item[$(i)$] For each block $P_k$ and $i,j \in P_k$ we have $\gamma_{i,j} \leq 1.2 \cdot \varepsilon$,
		\item[$(ii)$] Each block $P_k$ contains at least $\frac{1}{2\varepsilon}$ items,
		\item[$(iii)$] Each block $P_k$ contains at most $1/\varepsilon$ items,
		\item[$(iv)$] There are at most $2 M \varepsilon$ blocks.
	\end{itemize}	
\end{restatable}

\begin{proof}
	We will show that properties $(i)$ and $(ii)$ hold with probability at least $1-\delta$, and note that $(iii)$ follows directly from the algorithm and that $(iv)$ follows from $(ii)$. 
	
	Let $C$ be the event that the set $\CC$ returned by $\proc {get-net}$ is not an $\frac{\varepsilon}{2}$-net for $\mu$, and let $\MM$ be the set of $M$ items sampled. Similarly to in the proof of \cref{lmm:ClusteringWorksNew}, let $E_{i,j}$ be the event $\{S_{i,j}, \gamma_{ij}>0.6 \varepsilon \} \cup \{ S_{i,j}^c,\gamma_{ij} < 0.5 \varepsilon \}$, where $S_{ij}$ is the event that routine $\proc{similar}(i,j,0.6\varepsilon,\delta/(4M|\CC|))$ returns \const{similar}. Intuitively, the event $E_{i,j}$ happens when $\proc {similar}$ returns what it shouldn't have. Furthermore, let $E = \bigcup_{c \in \CC} \bigcup_{j \in \MM} E_{c,j}$.
	
	Let $F$ be the event that for some block $P_k$ there exists $i,j \in P_k$ such that $\gamma_{ij} > 1.2 \varepsilon$, and let 
	$B$ be the event that all blocks in the partition have size at least $\frac{1}{2\varepsilon}$. Therefore, $B^c = \bigcup_{k} B_k^c$, where $B_k^c$ is the even that $P_k$ has size less than $\frac{1}{2\varepsilon}$.
	
	Since event $F$ guarantees condition $(i)$ and event $B$ guarantees condition $(ii)$ it suffices to show that
	$$
	\Prob \left(F^c \cup B^c \right) \leq \delta.
	$$
	
	We will do so by conditioning on $C^c$ and $E^c$ where after a couple of union bounds we arrive at
	$$
	\Prob \left(F^c \cup B^c \right)
	\leq
	\Prob \left(F^c  \mid C^c , E^c \right)
	+
	\Prob \left( B^c \mid C^c , E^c \right)
	+
	\Prob \left( C \right)
	+
	\Prob \left(E \right),
	$$
	and showing
	\begin{itemize}
		\item[$(A)$] $\Prob \left(B^c  \mid C^c , E^c \right) \leq \delta/2$,
		\item[$(B)$] $\Prob \left(F^c \mid C^c , E^c \right) = 0$,
		\item[$(C)$] $\Prob \left(E \right) \leq \delta/4$, and
		\item[$(D)$] $\Prob \left(C \right) \leq \delta/4$
	\end{itemize}
	completes the proof.
	
	\emph{Proof of $(A)$}: Note that $B^c = \bigcup_{c \in \CC} B_c^c$, where $B_c$ is the event that the block $P_{k_c}$ 
	constructed with $c \in \CC$ as reference has size at least $\frac{1}{2 \varepsilon}$. 
	
	Note that the event $B_c$ happens whenever at least $\frac{1}{2} M  \left( \frac{\epsilon}{12} \right)^d$ 
	(which is at least $1/\varepsilon$) items are added to $P_{k_c}$. This is because in this event, $P_{k_c}$ will
	receive more than $1/\varepsilon$ items assigned to it and hence the algorithm will break it up in to smaller block but 
	each of them will be of size at least $\frac{1}{2 \varepsilon}$.
	Hence
	$$
	\Big\{	\sum_{n=1}^M \mathbf{1}_{X_{c,n}}	>\frac{1}{2} M  \left( \frac{\epsilon}{12} \right)^d \Big\}	\subset B_c,
	$$
	where $X_{c,n}$ is the event that the $n^{th}$ item sampled ends up in $P_{k_c}$. We will show how
	\begin{equation}
	\label{eq:MinBlockYieldingProbability}
	\Prob\left(X_{c,n} \mid C^c, E^c \right) \geq (\varepsilon/12) ^d
	\end{equation}
	allows us to prove $(A)$, and we then prove \cref{eq:MinBlockYieldingProbability}.
	Note that the $\{X_{c,n} \}_n$ are independent, and hence the sum $\sum_{n=1}^M \mathbf{1}_{X_{c,n}}$ stochastically dominates the sum $\sum_{n=1}^M Y_{c,n},$  where each $Y_{c,n}$ is an independent Bernoulli random variable of parameter $(\varepsilon/12)^d$. By the Chernoff bound we then get
	$$
	\Prob
	\left(
	B^c_c \mid C^c,E^c
	\right)
	\leq
	\Prob
	\left(
	\sum_{n=1}^M Y_{c,n}
	\leq 
	\frac{1}{2} M  \left( \frac{\epsilon}{12} \right)^d
	\right)
	\leq
	\exp
	\left(
	-\frac{1}{12} M \left(\frac{\varepsilon}{12}\right)^d
	\right)
	\leq
	\frac{\delta}{2} \left( \frac{\varepsilon}{8} \right)^d,
	$$
	where the last inequality is due to $M \geq 12 \left(\frac{12}{\varepsilon}\right)^d \ln \left(\frac{2}{\delta} \left( \frac{8}{\varepsilon} \right)^d \right)$. Hence we arrive at
	$$
	\Prob
	\left(
	B^c \mid C^c, E^c
	\right)
	\leq
	\sum_{c \in \CC} 
	\Prob
	\left(
	B^c_c \mid C^c, E^c
	\right)
	\leq
	\left(\frac{8}{\varepsilon} \right)^d
	\frac{\delta}{2}
	\left( \frac{\varepsilon}{8} \right)^d
	=
	\delta/2,
	$$
	as we wished.
	
	\emph{Proof of \cref{eq:MinBlockYieldingProbability}}: We can lower bound $\Prob\left(X_{c,n} \mid C^c, E^c \right)$ as
	\begin{equation*}
	\label{eqGuaranteeBlockSize}
	\tag{$\star$}
	\Prob\left(X_{c,n} \mid C^c, E^c \right)
	\geq
	\Prob\left(X_{c,n} \mid C^c, E^c , \gamma_{c,j_n} \leq \varepsilon/2 \right)
	\Prob_{j_n} \left( \gamma_{c,j_n} \leq \varepsilon/2 \right),
	\end{equation*}
	where $j_n \in \MM$ is the $n^{th}$ item drawn during $\proc {make-partition}$.
	
	Now note that the event $X_{c,n}$ occurs when 
	(a) $\proc {similar}(c,j_n, 0.6\varepsilon,\delta/(4 M |\CC|))$ returns $\const{true}$, 
	(b) $c$ is chosen uniformly at random among the other items in $\CC$ that are also similar $j_n$. 
	
	Conditioning on $E^c, C^c,$ and $\gamma_{c,j_n} \leq \varepsilon/2$ guarantees that 
	$\proc {similar}(c,j_n,0.6\varepsilon,\delta/(4 M |\CC|))$ returns $\const{true}$, and hence $\Prob\left(X_{c,n} \mid C^c, E^c , \gamma_{c,j_n} \leq \varepsilon/2 \right) = 1/K,$ where $K$ is the number of items in $\CC$ for which $\proc {similar}$ also returned $\const{true}$. By \cref{clm:DoublingDimensionClaims} (with $r \leftarrow 0.6 \varepsilon$ and $\varepsilon \leftarrow \varepsilon/2$ in the Proposition), we get that $K \leq (\frac{8}{\varepsilon})^d \mu \left(B \left(c, \varepsilon/2 \right) \right) \left(\frac{5}{4}+\frac{6}{5} \right)^d  \leq (\frac{12}{\varepsilon})^d \mu \left(B \left(c, \varepsilon/2 \right) \right) $. Now noting that in \cref{eqGuaranteeBlockSize} $\Prob_{j_n} \left( \gamma_{c,j_n} \leq \varepsilon/2 \right) = \mu \left(B \left(c, \varepsilon/2 \right) \right)$, we arrive at
	
\begin{align}
	\Prob\left(X_{c,n} \mid C^c, E^c \right)
	& \geq
	\Prob\left(X_{c,n} \mid C^c, E^c , \gamma_{c,j_n} \leq \varepsilon/2 \right)
	\Prob_{j_n} \left( \gamma_{c,j_n} \leq \varepsilon/2 \right) \nonumber \\
	& \geq
	\frac{1}{
		(\frac{12}{\varepsilon})^d \mu \left(B \left(c, \varepsilon/2 \right) \right)}
	\mu \left(B \left(c, \varepsilon/2 \right) \right)
	~=
	\left(\frac{\varepsilon}{12} \right)^d,
\end{align}
	which proves \cref{eq:MinBlockYieldingProbability} and hence $(A)$.

	\emph{Proof of $(B)$}: The event $B^c$ happens when for some block $P_k$ there are items $i,j \in P_k$ such that $\gamma_{ij} > 1.2 \varepsilon$. Conditioning on $C^c$, however, this can only happen if $\gamma_{c_k,j} > 0.6 \cdot \varepsilon$ (or $\gamma_{c_k,i} > 0.6 \cdot \varepsilon$) but $\proc {similar}(c_k,j,0.5 \varepsilon,\delta')$ returned \const{similar} nevertheless. Conditioning on $E^c$, however, that cannot happen and we get $\Prob \left(B^c \mid C^c , E^c \right) = 0$.
	
	\emph{Proof of $(C)$}: By \cref{lmm:PreIsSimilarNew}, $\Prob(E_c,j) \leq \frac{\delta}{4} \frac{1}{(4/\varepsilon)^d}\frac{1}{M}$ and hence
	$$
	\Prob(E) 
	\leq
	\sum_{c \in \CC} \sum_{j \in \MM} 
	\Prob(E_{i,j})
	\leq
	\frac{\delta}{4} \frac{|\CC|}{(4/\varepsilon)^d}\frac{M}{M}
	\leq
	\delta/4, 	
	$$
	where the last inequality is due to $| \CC | \leq (4/\varepsilon)^d$, as guaranteed in \cref{lmm:ClusteringWorksNew}.
	
	\emph{Proof of $(D)$}: This follows directly from \cref{lmm:ClusteringWorksNew}.
\end{proof}

\subsection{Sufficient Exploration}
\label{sec:suffExploration}

During epoch $\tau$ the algorithm uses any given recommendation to be an explore recommendation with probability $\varepsilon_\tau$. In the Lemma below, we should that during each epoch there are enough explore recommendations for the procedure $\proc{make-partition}$ to terminate.

\begin{restatable}{lmm}{lmmEnoughTimeForListNew}
	\label{lmm:EnoughTimeForListNew}
	With probability at least $1-\varepsilon_{\tau+1}$, during the $\tau^{th}$ epoch the algorithm has enough \emph{explore} recommendations for $\proc{make-partition} \left(M_{\tau+1}, \varepsilon_{\tau+1}, \varepsilon_{\tau+1}\right)$ to terminate.
\end{restatable}
\begin{proof}
	It suffices to prove the following two facts:
	\begin{itemize}
		\item[$(A)$] the number of times \emph{explore} is required for $\proc{make-partition} \left(M_{\tau+1}, \varepsilon_{\tau+1}, \varepsilon_{\tau+1}\right)$ to terminate is at most $\frac{1}{2} \varepsilon_{\tau} D_{\tau} N$ for each $\tau$, and
		\item[$(B)$] with probability at least $1-\varepsilon_{\tau+1}$ we have that \emph{explore} will be called at least $\frac{1}{2} \varepsilon_{\tau} D_{\tau}$ times.
	\end{itemize}
	\emph{Proof of $(A)$:} Let us denote by $MP(\tau+1)$ the number of explore calls required for the routine $\proc{make-partition} \left(M_{\tau+1}, \varepsilon_{\tau+1}, \varepsilon_{\tau+1}\right)$ to terminate. Then we want to show that $MP(\tau+1) \leq \frac{1}{2} \varepsilon_\tau D_\tau N$, or equivalently that
	\begin{equation*}
	\label{eqMPSufficientCondition}
	\tag{$\bigstar$}
	\frac{2}{\varepsilon} MP(\tau+1) / D_\tau
	\leq
	N
	.
	\end{equation*}
	
	Note that (which we soon check) $\proc {make-partition}(M_\tau,\varepsilon_\tau,\varepsilon_\tau)$ makes at most
	\begin{equation}
	MP(\tau+1)
	\leq
	\left(\frac{8}{\varepsilon_{\tau+1}}\right)^{d+1}
	4 \cdot 630 \left(d+1\right)^3 M_{\tau+1}
	\ln^2 \left(\frac{8}{\varepsilon_{\tau+1}}\right)
	\ln \left(M_{\tau+1}\right)
	\end{equation}
	recommendations, and since $M_\tau \triangleq \frac{2^{\max(8,3.5d)}}{\nu} \frac{(3d+1)}{\varepsilon_\tau^{d+2}} \ln (\frac{2}{\varepsilon_\tau})$ and $M_{\tau+1} \leq M_\tau \cdot 2^{d+2}$ we get
	\begin{equation}
	\frac{4}{\varepsilon_\tau} \frac{1}{M_\tau} MP(\tau+1)
	\leq
	\left(\frac{16}{\varepsilon_\tau}\right)^{d+1}
	4 \cdot 630 \left(d+1\right)^3 2^{d+2}
	\ln^2 \left(\frac{16}{\varepsilon_{\tau}}\right)
	\ln \left(\frac{2^{3d+13}}{\nu \varepsilon_\tau^{d+1}} \ln(1/\varepsilon_\tau)\right)
	\end{equation}
	\begin{equation}
	\leq
	2^{5d+12} \cdot 630 (d+1)^3 \left(\frac{1}{\varepsilon_\tau}\right)^{d+2}
	\ln^2 \left(\frac{16}{\varepsilon_{\tau}}\right)
	\ln \left(\frac{2^{3d+13}}{\nu \varepsilon_\tau^{d+2}}\right).
	\end{equation}
	It is simple to show that $\ln^2 \left(\frac{16}{\varepsilon_{\tau}}\right)
	\ln \left(\frac{2^{3d+13}}{\nu \varepsilon_\tau^{d+2}} \right)  \leq \frac{2^6}{\nu} (2d+11)(d+2)\frac{1}{\varepsilon_\tau^3}$, which we can use to further bound $\frac{4}{\varepsilon_\tau} \frac{1}{M_\tau} MP(\tau+1)$ as
	$$
	\frac{4}{\varepsilon_\tau} \frac{1}{M_\tau} MP(\tau+1)
	\leq
	\left(2^{5d+12} \cdot 630 (d+1)^3 \left(\frac{1}{\varepsilon_\tau}\right)^{d+2} \right)
	\left(\frac{2^6}{\nu} (2d+11)(d+2)\frac{1}{\varepsilon_\tau^3}\right)
	$$
	$$
	=
	\frac{2^{5d+18}}{\nu} \cdot 630 (2d+11)(d+2)^4 \frac{1}{\varepsilon_\tau^{d+5}},
	$$
	and since $\varepsilon_\tau \geq \varepsilon_N = \left( \frac{2^{5d+18}}{\nu} \cdot 630 (2d+11)(d+2)^4 \frac{1}{N} \right)^{\frac{1}{d+5}}$,  we get that \cref{eqMPSufficientCondition} is satisfied and we are done.
	
	\emph{Proof of $(B)$:} \emph{explore} is called with probability $\varepsilon_{\tau}$ at each of the $D_{\tau} N$ recommendations of epoch $\tau$. Hence, all we need to show is that $\mathbb{P}\left( \mathrm{Bin} \left(D_{\tau} N ,\varepsilon_{\tau} \right) \leq \frac{1}{2} \varepsilon_{\tau} D_{\tau} N \right) $ is at most $\varepsilon_{\tau+1}$. This follows from the Chernoff bound:
	$$
	\mathbb{P}
	\Big(
	\mathrm{Bin} \left(D_{\tau} N,\varepsilon_{\tau} \right) < \frac{1}{2} \varepsilon_{\tau} D_{\tau} N
	\Big)
	=
	\mathbb{P}
	\Big(
	\mathrm{Bin} \left( D_{\tau} N, \varepsilon_{\tau}\right) < 
	\left(1-0.5 \right) 
	\EE \left[ \mathrm{Bin} \left(D_{\tau} N, \varepsilon_{\tau} \right) \right]
	\Big)
	$$
	$$
	\leq
	\exp \left( - \frac{0.5^2}{2+0.5^2}  \EE \left[ \mathrm{Bin} \left(D_{\tau} N,\varepsilon_{\tau} \right) \right] \right)
	=
	\exp \left(- \frac{1}{9}  \varepsilon_{\tau} D_{\tau} N \right)
	\leq \varepsilon_{\tau+1},
	$$
	where the second to last inequality follows from $D_{\tau} \geq \frac{9}{\varepsilon_{\tau}} \ln( \frac{1}{\varepsilon_{\tau+1} })$.
\end{proof}

\section{Quick Recommendations Lemma}

In \cref{sec:algorithm} we described that the algorithm, which starts recommending to a user as soon as it knows of one item that the user likes. Below we show that indeed shortly after the beginning of the epoch the slope of the regret is small. 

\begin{restatable}[Quick Recommendations Lemma]{lmm}{lmmColdStart}
	\label{lmm:ColdStart}
	For $\tau \geq 1$ let $\RR^{(\tau)}(T) = \frac{1}{N} \sum_{t=T_\tau}^{T_\tau +TN}\frac{1}{2} (1-L_{U_t,I_t}) $ denote the number of bad recommendations made to users during the first $TN$ recommendations of epoch $\tau$. Then we have
	\begin{align}\label{eq:regret}
	\EE \left[ \mathcal{R}^{(\tau)}(T) \right]
	\leq
	\frac{148 \varepsilon_\tau}{\nu}  T
	\end{align}
	whenever $T \in [T_{\mathrm{min,\tau}} ,D_\tau]$ and where $T_{\mathrm{min,\tau}} \triangleq \frac{12}{\varepsilon_\tau} \ln(\frac{1}{\varepsilon_\tau})$. For $T<T_{\mathrm{min,\tau}} $, we trivially have $\EE \left[ \mathcal{R}^{(\tau)}(T) \right] \leq T$.
\end{restatable}
\begin{proof}
	Let $\RR'^{(\tau)}(T)$ denote the number of bad recommendations made to users during the first $TN$ \emph{exploit} recommendations of epoch $\tau$. Then, since the expected number of explore recommendations by time $TN$ of epoch $\tau$ is $\varepsilon_\tau TN$, we get that
	\begin{equation}
	\EE \left[\RR^{(\tau)}(T)\right]
	\leq
	\varepsilon_\tau T + 
	\EE \left[\RR'^{(\tau)}(T)\right].
	\end{equation}
	
	Furthermore, as described in the algorithm, during the epoch $\tau-1$ the algorithm spends a small fraction of the recommendations, in the explore part, to create a partition $\{P_k\}$ (which we call $\{P_k^{(\tau)}\}$ in the pseudocode) of $M_\tau$ random items to be exploited during epoch $\tau$. Let $\mathcal{E}_\tau$ be the event that the partition $\{P_k^{(\tau)}\}$ to be used during epoch $\tau$ satisfies the conditions specified in \cref{lmm:partitionConditions} (with $M=M_\tau$ $\varepsilon=\varepsilon_\tau$, $\delta=\varepsilon_\tau$). Then we get
	\begin{equation}
	\EE \left[ \RR^{(\tau)}(T) \right] 
	\leq
	\varepsilon_\tau T + 
	\EE \left[\RR'^{(\tau)}(T)\right]
	\leq
	2\varepsilon_\tau T + 
	\EE \left[\RR'^{(\tau)}(T) \mid \mathcal{E}_\tau \right],
	\end{equation}
	where the last inequality is due to \cref{lmm:partitionConditions}, which guarantees that $\Prob(\mathcal{E}_\tau) \leq \varepsilon_\tau$.
	
	%
	For the remaining of the proof, we will show that $\EE \left[\RR'(T)^{(\tau)}(T)\right] \leq \frac{45}{\nu}\varepsilon_\tau T$. We will do so by first rewriting in terms of the number of bad exploit recommendations to each user
	$$
	\EE \left[ \RR'^{(\tau)} \mid \mathcal{E}_\tau \right] 
	\leq 
	\frac{1}{N}
	\EE \left[ \sum_{u} \RR'^{(\tau)}(T) \mid \mathcal{E}_\tau \right],
	$$
	where $\RR_u'^{(\tau)}(T)$ is the number of bad recommendations made to user $u$ during the first $TN$ exploit recommendations of epoch $\tau$. We will now bound the latter term by conditioning on a nice property of users (which we will characterize by the event $g_{u,T}$), and showing that this property holds for most users. Let $g_{u,T}$ be the event that user $u$ has tried at most $16 \frac{T}{D_\tau} P^{(\tau)}$ blocks during the first $TN$ recommendations of epoch $\tau$ (we omit $\tau$ in the notation of $g_{u,T}$ since it is clear from the context here). Here we use notation $P^{(\tau)} =  |\{P_k^{(\tau)}\}|$ to denote the total number of blocks in the partition for epoch $\tau$. Then we get
	$$
	\frac{1}{N}
	\EE \left[ \sum_{u} \RR_u'^{(\tau)}(T) \mid \mathcal{E}_\tau \right]
	\leq
	\frac{1}{N}
	\sum_{u} \EE \left[ \RR_u'^{(\tau)}(T) \mid \mathcal{E}_\tau , g_{u,T} \right]
	+
	T
	\frac{1}{N}
	\sum_{u} \Prob \left(g_{u,T}^c \mid \mathcal{E}_\tau \right).
	$$
	
	We dedicate \cref{lmm:blockExplorationBound} to showing that $  \frac{1}{N}
	\sum_{u} \Prob \left(g_{u,T}^c \mid \mathcal{E}_\tau \right) \leq \frac{42}{\nu} \varepsilon_\tau$. Hence, it suffices to show that for each $ T> T_{\mathrm{min},\tau}$ we have
	\begin{equation*}
	\label{eqColdStartCondition}
	\tag{$\bigstar$}
	\frac{1}{N}
	\sum_{u} \EE \left[ \RR_u'^{(\tau)}(T) \mid \mathcal{E}_\tau , g_{u,T} \right] \leq \frac{104}{\nu} \varepsilon_\tau T,
	\end{equation*}
	which we prove now. We will first rewrite the regret by summing over the number of bad recommendations due to each of the blocks as
	$$
	\sum_{u} \EE \left[ \RR_u'^{(\tau)}(T) \mid \mathcal{E}_\tau , g_{u,T} \right]
	=
	\sum_{u} \EE \left[ \sum_k W_{u,k,T} \mid \mathcal{E}_\tau , g_{u,T} \right],
	$$
	where $W_{u,k,T}$ is the random variable denoting the number of bad exploit recommendations to user $u$ from block $P_k$ among the first $TN$ exploit recommendations of epoch $\tau$. We can further rewrite this as
\begin{align}\label{eq:zz1}
	\frac{1}{N} \sum_{u} \EE \left[ \sum_k W_{u,k,T} \mid \mathcal{E}_\tau , g_{u,T} \right]
	\leq
	\frac{1}{N} \sum_{u} \sum_k \EE \left[ W_{u,k,T} \mid \mathcal{E}_\tau , g_{u,T}, s_{u,k,T} \right] \Prob \left(s_{u,k,T} \mid  \mathcal{E}_\tau, g_{u,T}\right),
\end{align}
	where $s_{u,k,T}$ denotes the event that by time $T$ user $u$ has sampled an item from block $P_k$. Note that the reason why the natural term $\EE \left[ W_{u,k,T} \mid \mathcal{E}_\tau , g_{u,T}, s^c_{u,k,T} \right] \Prob \left(s_{u,k,T}^c \mid  \mathcal{E}_\tau , g_{u,T}\right)$ is absent from the expression above is because $\EE \left[ W_{u,k,T} \mid \mathcal{E}_\tau , g_{u,T}, s^c_{u,k,T} \right] = 0$ since the user hasn't sampled an item from the block.
	
	Now note that by conditioning on $g_{u,T}$ , we know that user $u$ has sampled at most $16 \frac{T}{D_\tau}P^{(\tau)}$ blocks. 
	Now given $g_{u, T}$ as well as $\mathcal{E}_\tau$, the indices of the sampled blocks are not revealed. Let $K$ be a random variable that 
	selects one of the indices of the blocks uniformly at random. Then, it follows that with respect to randomness in $K$, 
	\begin{equation}
	\Prob \left(s_{u,K,T} \mid \mathcal{E}_\tau , g_{u,T}\right)
	\leq
	\frac{16 \frac{T}{D_\tau} P^{(\tau)}}{P^{(\tau)}}
	=
	16 \cdot \frac{T}{D_\tau}.
	\end{equation}
We can re-write \eqref{eq:zz1} in this notation and apply the above discussed bound to obtain
\begin{align}\label{eq:zz2}
	\frac{1}{N} \sum_{u} \EE \left[ \sum_k W_{u,k,T} \mid \mathcal{E}_\tau , g_{u,T} \right]
	& \leq
	\frac{P^{(\tau)}}{N} \sum_{u} \EE \left[ W_{u,K,T} \mid \mathcal{E}_\tau , g_{u,T}, s_{u,K,T} \right] \Prob \left(s_{u,K,T} \mid  \mathcal{E}_\tau, g_{u,T}\right) \nonumber \\
	& \leq
	\frac{P^{(\tau)}}{N} \sum_{u} \EE \left[ W_{u,K,T} \mid \mathcal{E}_\tau , g_{u,T}, s_{u,K,T} \right] \frac{16T}{D_{\tau}} \nonumber \\ 
	& = \frac{16T}{ND_\tau} \sum_{u}\sum_{k} \EE \left[ W_{u,k,T} \mid \mathcal{E}_\tau , g_{u,T}, s_{u,k,T} \right].
\end{align}
The right hand side above can be bounded as 
\begin{align}
	& \underbrace{\leq}_{\cref{lmm:partitionProperty}}
	\frac{16  T}{N D_\tau} \big( \sum_k (1+1.2 \varepsilon_\tau |P_k^{(\tau)}|)N \big) \nonumber \\
	& \underbrace{\leq}_{\cref{lmm:partitionConditions}}
	52\varepsilon_\tau T \frac{M_\tau}{D_\tau}
	\overbrace{=}^{D_\tau = \frac{\nu}{2} M_\tau}
	\frac{104}{\nu} \varepsilon_\tau T.
\end{align}
To see the above two inequalities, consider the following. The first inequality follows from \cref{lmm:partitionProperty}
by realizing that each block, $P_k^{(\tau)}$ corresponds to collection of items such that for any $i, j \in P_k^{(\tau)}$,
we have $\gamma_{ij} < 1.2\varepsilon_\tau$. The second inequality can be argued
	as: from \cref{lmm:partitionConditions}, $1/2\varepsilon_\tau \leq |P_k^{(\tau)}| \leq 1/\varepsilon_\tau$ and hence
	\begin{align*}
	\sum_k (1+1.2\varepsilon_\tau |P_k^{(\tau)}|) & \leq \sum_k (3.2 \varepsilon_\tau |P_k^{(\tau)}|) \nonumber \\
	& \leq 3.2 \varepsilon_\tau (\sum_k |P_k^{(\tau)}|) \nonumber \\
	& = 3.2 \varepsilon_\tau M_\tau.
	\end{align*}
This, along with simple calculation, completes the proof of \cref{eqColdStartCondition}.
\end{proof}

The lemma below was used in \cref{lmm:ColdStart}. Informally, it says that our recommendation policy, which recommends the whole block to a user after the user likes an item in the block, succeeds in finding most likable items to recommend and in not recommending many bad items.
\begin{restatable}[Partition Lemma]{lmm}{lmmPartitionProperty}
	\label{lmm:partitionProperty}
	Let $P_k$ be a set of items such that for each $i,j \in P_k$ we have $\gamma_{ij} < \varepsilon$, and consider the usual recommendation policy that \proc{item-item-cf} uses during its ``exploit" steps (where when user $u$ samples a random item $i \in_R P_k$, only if $u$ likes $i$ will $u$ be recommended the remaining items). Let $s_{u,k}$ be the event that user $u$ has sampled an item from $P_k$, let $W_{u,k}$ ($W$ for wrong) denote the number of wrong recommendations made to $u$ from $P_k$, and let $A_{u,k}$ ($A$ for absent) denote the number of items in $P_k$ that $u$ likes that are not recommended to $u$. Then we get
	$$
	\sum_u \EE \left[ A_{u,k} + W_{u,k} \mid s_{u,k} \right] \leq  \big(1+\varepsilon |P_k| \big) N.
	$$
\end{restatable}

\begin{proof}
	For each block $P_k$ and user $u$, and let $\ell_{u,k} = | \{ i \in P_k \mid L_{u,i}=+1  \} |$ 
	denote the number of items in $P_k$ that $u$ likes. Note that 
	$\EE \left[ A_{u,k} \mid s_{u,k}  \right] = \ell_{u,k} \cdot \frac{ \left(|P_k| -\ell_{u,k} \right)}{|P_k|} $ and 
	$\EE  \left[ W_{u,k} \mid s_{u,k} \right] = \left(|P_k| -\ell_{u,k} \right) \cdot \frac{\ell_{u,k}}{|P_k|}+1 \cdot \frac{\left(|P_k| -\ell_{u,k} \right)}{|P_k|} $. This is because with probability $\ell_{u,k}/|P_k|$ user $u$ will sample an item from $P_k$ that $u$ likes and will then be recommended 
	$\left( |P_k| - \ell_{u,k} \right)$ bad items, and with probability $(|P_k|-\ell_{u,k})/|P_k|$ the first item recommended to $u$ is bad.
	Likewise, with probability $\left( |P_k| - \ell_{u,k} \right)/|P_k|$ 
	the user will sample an item that the user dislikes, and then fail to be recommended $\ell_{u,k}$ items that the user likes. Hence we have that
	\begin{equation}
	\EE \left[ \sum_u  A_{u,k}+W_{u,k} \mid s_{u,k} \right] 
	=
	\underbrace{
		\sum_u \frac{\left(|P_k| -\ell_{u,k} \right)}{|P_k|}
	}_{\leq N}
	+
	2 \sum_u \frac{\ell_{u,k} \left(|P_k| -\ell_{u,k} \right)}{|P_k|}.
	\end{equation}
Now, 
%
%
	\begin{align}
2 \sum_u \frac{\ell_{u,k} \left(|P_k| -\ell_{u,k} \right)}{|P_k|}
	& \overbrace{=}^{\text{by definition}}
	\frac{2}{|P_k|}
	\sum_u
	\sum_{i,j \in P_k}
	\mathbf{1}_{L_{u,i} \ne L_{u,j} } \nonumber \\
	& =
	\frac{2}{|P_k|}
	\sum_{i,j \in P_k}
	\underbrace{
		\sum_u 
		\mathbf{1}_{L_{u,i} \ne L_{u,j}}
	}_{=\gamma_{ij} \cdot N} \nonumber \\
	& \underbrace{\leq}_{\gamma_{ij}<\varepsilon}
	\frac{2}{|P_k|}
	\sum_{i,j \in P_k}
	\varepsilon N \nonumber \\
	& \leq
	\frac{2}{|P_k|}
	{ |P_k| \choose 2} \varepsilon N \leq \varepsilon |P_k| N.
	\end{align}
Putting it all together we get
	\begin{equation}
	\EE \left[ \sum_u  A_{u,k}+W_{u,k} \mid s_{u,k} \right] 
	=
	\underbrace{
		\sum_u \frac{\left(|P_k| -\ell_{u,k} \right)}{|P_k|}
	}_{\leq N}
	+
	\underbrace{
		2 \sum_u \frac{\ell_{u,k} \left(|P_k| -\ell_{u,k} \right)}{|P_k|}
	}_{\varepsilon |P_k| N}
	=
	\big(\varepsilon |P_k| +1 \big) N.
	\end{equation}
\end{proof}

The lemma below was also needed in the proof of \cref{lmm:ColdStart}.
\begin{restatable}[Auxiliary Claim]{lmm}{lmmBlockExplorationBound}
	\label{lmm:blockExplorationBound}
	Consider an arbitrary epoch $\tau$, and let $g_{u,T}$ be the event that by the $(T N)^{th}$ exploit recommendation of epoch $\tau$ user $u$ has tried at most $16 \frac{T}{D_\tau} |\{P_k^{(\tau)}\}|$ 
	blocks from the partition $\{P_k^{(\tau)}\}$ constructed during the $\proc {make-partition}(M_\tau,\varepsilon_\tau,\varepsilon_\tau)$ of the previous epoch, and let $\mathcal{E}_\tau$ be the event that $\{P_k^{(\tau)}\}$ satisfies the conditions specified in \cref{lmm:partitionConditions}. Then 
	$$
	\frac{1}{N} \sum_u \Prob \left( g_{u,T}^c \mid \mathcal{E}_\tau \right) \leq \frac{42}{\nu} \varepsilon_\tau
	$$ 
	holds for any $T \in \big( \frac{12}{\varepsilon_\tau} \ln (1/\varepsilon_\tau),D_\tau \big]$.
\end{restatable}

\begin{proof}
Let us consider few definitions:
\begin{itemize}
\item[1.] Let $N_{u,T}$ be the event that by the $TN^{th}$ exploit recommendation user $u$ has been recommended at most $1.1 T$ items, \emph{and} that $u$ likes at least $0.9 \nu M_\tau$ among the items in $\{P_k^{(\tau)}\}$.
\item[2.] Let $H_{u,T}$ be the event that by the $TN^{th}$ exploit recommendation there are still at least $\frac{\nu}{5} M_\tau $ items liked by $u$ in blocks that haven't been sampled by $u$. 
\end{itemize}
Then we get
	$$
	\Prob \left(  g_{u,T}^c \mid \mathcal{E}_\tau \right)
	\leq
	\Prob \left( g_{u,T}^c \mid H_{u,T}, N_{u,T}, \mathcal{E}_\tau  \right)
	+
	\Prob \left( H_{u,T}^c \mid N_{u,T}, \mathcal{E}_\tau \right)
	+
	\Prob \left( N_{u,T}^c \mid \mathcal{E}_\tau  \right).
	$$
	
	We will show that for $T \in \big( \frac{12}{\varepsilon_\tau} \ln (1/\varepsilon_\tau),D_\tau \big]$
	\begin{itemize}
		\item[$(A)$] $	\Prob \left( g_{u,T}^c \mid H_{u,T}, N_{u,T}, \mathcal{E}_\tau  \right) \leq \varepsilon_\tau$
		\item[$(B)$] $ \frac{1}{N} \sum_u \Prob \left( H_{u,T}^c \mid N_{u,T}, \mathcal{E}_\tau \right) \leq \frac{40}{\nu} \varepsilon_\tau$
		\item[$(C)$] $\Prob \left( N_{u,T}^c \mid \mathcal{E}_\tau  \right) \leq \varepsilon_\tau$,
	\end{itemize}
	from which the lemma follows.
	
	\emph{Proof of $(A )$}: Note that by conditioning on $N_{u,T}$ we know that there were at least $0.2 \nu M_\tau$ items likable to $u$ for each $t \leq T$. Now let $\ell_{u,k}$ denote the event that $u$ likes the first item sampled from block $P_k^{(\tau)}$. Then we have that
	\begin{equation}
	\Big\{ 
	g_{u,T}^c , \mathcal{E}_\tau, N_{u,T} , H_{u,T}
	\Big\}
	\subseteq
	\bigg\{
	\sum_{n=1}^{16 \frac{T}{D_\tau} |\{P_k^{(\tau)}\}|}
	|P_{k_n}| 
	\cdot
	\mathbf{1}_{\ell_{u,{k_n}} } \leq 1.1T
	,
	\mathcal{E}_\tau,
	N_{u,T}, H_{u,T}
	\bigg\},
	\end{equation}
	where $P_{k_n}$ is the $n^{th}$ block sampled by $u$. This in turn gives us
	\begin{equation}
	\Prob 
	\bigg( 
	g_{u,T}^c \mid \mathcal{E}_\tau, N_{u,T}, H_{u,T} 
	\bigg)
	\leq
	\Prob
	\Bigg(
	\sum_{n=1}^{16 \frac{T}{D_\tau} |\{P_k^{(\tau)}\}|}
	|P_{k_n}| 
	\cdot
	\mathbf{1}_{\ell_{u,{k_n}} }
	\leq 1.1 T
	\mid
	\mathcal{E}_\tau, N_{u,T}, H_{u,T}
	\Bigg),
	\end{equation}
	and we will now prove that the latter is at most $\varepsilon_\tau$.
	
	First, note that by conditioning on $\mathcal{E}_\tau$, we are guaranteed that 
	$|P_{k_n}|	\geq 1/2\varepsilon_\tau$ 
	by \cref{lmm:partitionConditions}. Hence 
	\begin{align*}
	& \Prob
	\Bigg(
	\sum_{n=1}^{16 \frac{T}{D_\tau} |\{P_k^{(\tau)}\}|}
	|P_{k_n}| 
	\cdot
	\mathbf{1}_{\ell_{u,{k_n}} }
	\leq 1.1T
	\mid
	\mathcal{E}_\tau, H_{u,T}
	N_{u,T}
	\Bigg) \nonumber \\
	& \qquad \leq
	\Prob
	\Bigg(\sum_{n=1}^{16 \frac{T}{D_\tau} |\{P_k^{(\tau)}\}|}
	\mathbf{1}_{\ell_{u,{k_n}} }
	\leq 2.2 T \varepsilon_\tau 	\mid	\mathcal{E}_\tau, H_{u,T}, N_{u,T}	\Bigg).
	\end{align*}
	We will now show following two claims, which in turn implies (A) in light of the above discussion. 
	\begin{itemize}
		\item[$(A_1)$] For each $n$, $\Prob(\ell_{u,k_n} \mid 	\mathcal{E}_\tau, H_{u,T}, N_{u,T})$ stochastically dominates a Bernoulli random variable with parameter $0.1 \nu$.
		\item[$(A_2)$] Let $\{X_n\}$ be a set of independent Bernoulli random variables with parameter $0.1 \nu$. Then
		\begin{equation}
		\Prob \left(
		\sum_{n=1}^{16 \frac{T}{D_\tau}  |\{P_k^{(\tau)}\}|}
		X_n
		<
		2.2 T \varepsilon_\tau 
		\right)
		\leq
		\varepsilon_\tau
		\end{equation}
	\end{itemize}
	
	\emph{Proof of $(A_1)$}: Since we are conditioning on $H_{u,T}$, there are still $0.2 \nu M_\tau$ liked items in unsampled blocks, and since we are conditioning on $\mathcal{E}_\tau$, there are at most $2 M_\tau \varepsilon_\tau$ blocks
	and each of size at leaste $1/2\varepsilon_\tau$ and at most $1/\varepsilon_\tau$. It can be easily argued that the setting in which sampling a random block yields a likable item is least likely is when all likable items are in the largest blocks. By $\mathcal{E}_\tau$, we can ``fit'' $0.2 \nu M_\tau$ 
in at least $0.2 \nu M_\tau/(1/\varepsilon_\tau) = 0.2\varepsilon_\tau \nu M_\tau$ blocks. Therefore, 
	\begin{equation}
	\Prob(\ell_{u,k_n} \mid 	\mathcal{E}_\tau, H_{u,T}, N_{u,T})
	\geq
	\frac{0.2\varepsilon_\tau \nu M_\tau}{2 M_\tau \varepsilon_\tau }
	= 0.1 \nu.
	\end{equation}
	
	\emph{Proof of $(A_2)$}: Let $\{X_n\}$ be a set of independent Bernoulli random variables with parameter $0.1 \nu$.  Then
	the sum of $16 \frac{T}{D_\tau} |\{P_k^{(\tau)}\}|$ such i.i.d. random variables have average equal to 
	\begin{align}\label{eq:zz11}
	1.6\nu \frac{T}{D_\tau} |\{P_k^{(\tau)}\}| & \geq 1.6\nu \frac{T}{D_\tau} M_\tau \varepsilon_\tau \nonumber \\
	& = 3.2  T \varepsilon_\tau, 
	\end{align}
	where we have used the fact that since by $\mathcal{E}_\tau$ each block is of size at most $1/\varepsilon_\tau$ and number
	of items $M_\tau$, we have at least $M_\tau \varepsilon_\tau$ blocks and $D_\tau = M_\tau\nu/2$.
	Then
	\begin{equation}
	\Prob \left(
	\sum_{n=1}^{16 \frac{T}{D_\tau} |\{P_k^{(\tau)}\}|}
	X_n
	<
	2.2 T \varepsilon_\tau
	\right)
	\leq
	\Prob \left(
	\sum_{n=1}^{16 \frac{T}{D_\tau} |\{P_k^{(\tau)}\}|}
	X_n
	<
	\left(1-0.25\right) 
	\EE \left[\sum_{n=1}^{16 \frac{T}{D_\tau} |\{P_k^{(\tau)}\}|} X_n\right]
	\right),
	\end{equation}
	where the inequality is due to \eqref{eq:zz11}.
	 Now by the Chernoff bound we get
	\begin{equation}
	\exp \left(- \frac{0.25^2}{2+0.25}  
	\EE \left[\sum_{n=1}^{16 \frac{T}{D_\tau} |\{P_k^{(\tau)}\}|} X_n\right] \right)
	\leq
	\exp
	\bigg(-\frac{3.2T \varepsilon_\tau}{36}\bigg).
	\end{equation}
         Therefore, if $T  > \frac{12}{\varepsilon_\tau}\ln(1/\varepsilon_\tau)$, then the right hand side is less than $\varepsilon_\tau$, as desired. 
	
	\emph{Proof of $(B)$}: First, for each $u$ it is clear that $\Prob ( H_{u,T}^c \mid N_{u,T}) \leq  \Prob( H_{u,D_\tau}^c \mid N_{u,D_\tau} )$ holds, since by end of the epoch the user will have explored the most (recall $T \leq D_\tau$). Recall that under event $N_{u,D_\tau}$, user $u$ has been recommended at most $1.1 D_\tau$ items \emph{and} that at least $0.9 \nu M_\tau$ items in the partition are likable to $u$. For $H_{u, D_\tau}^c$
to happen, that is, for there to be at most $0.2 \nu M_\tau$ at the end of $D_\tau N$ exploit recommendation (overall), it must be that 
at least 
\begin{align}
0.9 \nu M_\tau - 0.2 \nu M_\tau - 1.1 D_\tau  & = 0.7 \nu M_\tau - 0.55 \nu M_\tau = 0.15 \nu M_\tau.
\end{align}
many items liked by user $u$ are ``wasted''.  Using notation from the proof of \cref{lmm:partitionProperty}, formally it can be written as 
\begin{align}	
\sum_k A_{k,u} \mathbf{1}_{s_{u,k,D_\tau} } \geq 0.15 \nu M_\tau.
\end{align}
	
	
By Markov's inequality we get that
\begin{align}
	\Prob\Big(\sum_k A_{k,u} \mathbf{1}_{s_{u,k,D_\tau}} \geq 0.15 \nu M_\tau \mid \mathcal{E}_\tau, N_{u,T}	\Big) 
	& \leq
	\Prob\Big(\sum_k A_{k,u} \geq 0.15 \nu M_\tau \mid \mathcal{E}_\tau, N_{u,T}	\Big)
	\nonumber \\ 
	& \stackrel{(a)}{\leq} 	\frac{\Prob\Big(\sum_k A_{k,u} \geq 0.15 \nu M_\tau \mid \mathcal{E}_\tau	\Big)}{\Prob\Big( N_{u,T}\Big)}
	\nonumber \\ 
	& \stackrel{(b)}{\leq} 	\frac{\Prob\Big(\sum_k A_{k,u} \geq 0.15 \nu M_\tau \mid \mathcal{E}_\tau	\Big)}{1-\varepsilon_\tau}
	\nonumber \\ 
	& \leq
	\frac{
		2 \EE
		\left[
		\sum_k A_{k,u}
		\mid
		\mathcal{E}_\tau \right]
	}
	{ 0.15 \nu M_\tau },	
\end{align}
	where inequality (a) uses the fact that $P(A | B) \leq P(A)/P(B)$;  inequality (b) uses the fact (C) (we note that (B) is NOT 
	used to prove (C) and hence there is no circularity of the argument); and last inequality uses the fact that $\varepsilon_\tau \leq 1/2$ for
	all $\tau$. We note that effectively, we have assumed away that $s_{u, k, D_\tau}$ has happened for all $k$ for the given user $u$. Therefore, 
	we will utilize \cref{lmm:partitionProperty} to bound the right hand side as follows:
\begin{align}
	\frac{1}{N}
	\sum_u
	\Prob
	\Big(
	\sum_k A_{k,u} \mathbf{1}_{s_{u,k,D_\tau}} \geq 0.35 \nu M_\tau
	\mid
	\mathcal{E}_\tau, N_{u,T}
	\Big)
	& \leq 
	\frac{1}{N}
	\sum_u\frac{
		2 \EE
		\left[
		\sum_k A_{k,u}
		\mid
		\mathcal{E}_\tau \right]
	}
	{ 0.15 \nu M_\tau }	\nonumber \\
	& \leq 
	\frac{2}{0.15\nu M_\tau N} \Big(\sum_k (1+\varepsilon_\tau |P_k^{\tau}| ) N\Big)\nonumber \\
	& \stackrel{(a)}{\leq} \frac{2}{0.15\nu M_\tau}\Big(\sum_k 3 \varepsilon_\tau |P_k^{\tau}| ) \Big)\nonumber \\
	& =  \frac{40\varepsilon_\tau}{\nu M_\tau}\Big(\sum_k |P_k^{\tau}|)\Big)  \nonumber \\
	& =  \frac{40 \varepsilon_\tau}{\nu}
	\end{align}
	where (a) uses the fact that under $\mathcal{E}_\tau$, $|P_k^{\tau}| \geq 1/2\varepsilon_\tau$ for all $k$, and we have
	used the fact that all partitions sum up to $M_\tau$.

	\emph{Proof of $(C)$}: The event $ N_{u,T}^c$ happens whenever a user $u$ has been recommended more than $1.1 D_\tau = 0.55 \nu M_\tau$ times by time $T$, or when the users likes less than $0.95 \nu M_\tau$. The probability that user $u$ has been recommended more than $0.55 \nu M_\tau$ items by time $T$ is greatest at $T=D_\tau$, and, by Chernoff bound, is
	\begin{equation}
	\mathbb{P}
	\left(
	\mid \mathrm{Bin} \left( N D_\tau ,\frac{1}{N} \right)  \geq 1.1 D_\tau
	\right)
	\leq
	\exp
	\left(
	-\frac{.1^2}{2+0.1} D_\tau
	\right),
	\end{equation}
	which, since $D_\tau> 210 \ln (\frac{2}{\varepsilon_\tau})$, is at most $\varepsilon_\tau/2$. 
	
	The probability that a user $u$ likes less than $0.9 \nu M_\tau$ among the $M_\tau$ items can be also bounded using a Chernoff bound:
	\begin{equation}
	\mathbb{P}
	\left(
	\mathrm{Bin} \left( M_{\tau}, \nu \right) < 0.9 \nu M_\tau
	\right)
	\leq
	\exp
	\left(
	-\frac{.1^2}{2+0.1} \frac{1}{2} M_\tau \nu
	\right)
	\leq
	\varepsilon_\tau/2,
	\end{equation}
	where the last inequality is due to $M_\tau \geq 210 \ln(\frac{2}{\varepsilon_\tau})$.
\end{proof}

\section{Proof of Main Results}

In the previous section we proved \cref{lmm:ColdStart}, which states that shortly after the beginning of each epoch, the expected regret of the algorithm becomes small. This allows us to prove our main results below.

\thmManyEpochsRegret*
\begin{proof}
	Recall that during the beginning of $\proc{item-item-cf}$ it runs the routine $\proc{make-partition}(M_1,\varepsilon_1,\varepsilon_1)$. This consumes at most 
	\begin{equation}
	MP(1)
	\triangleq
	\left(\frac{8}{\varepsilon_{1}}\right)^{d+1}
	4 \cdot 630 \left(d+1\right)^3 M_{1}
	\ln^2 \left(\frac{8}{\varepsilon_{1}}\right)
	\ln \left(M_{1}\right)
	\end{equation}
	recommendations (by \cref{lmm:partitionConditions}), and hence finishes in at most $T_{MP} \triangleq MP(1)/N$ time steps. For this initial exploratory period $T \leq T_{MP}$ we will bound the regret with the trivial bound $\RR(T) \leq T$.

	Let us now deal with the regime between $T_{min}$ and $T_{max}$. Recall that the target $\varepsilon_\tau$ used in the $\tau^{th}$ epoch is decreasing as $\frac{C}{2^\tau}$, until it plateaus at $\varepsilon_N$ when $\frac{C}{2^\tau} \leq \varepsilon_N$, where $C = \frac{\nu}{148 \cdot 20}$.
Hence 
	\begin{equation}
	\tau^* \triangleq \lceil \log_2 \frac{C}{\varepsilon_N} \rceil
	\end{equation} 
	is the first epoch in which $\varepsilon_N$ is used. For a function $g$ defined later, we will show that
	\begin{equation}
	T_{MP(1)} + \sum_{\tau'=1}^{\tau^*-1} D_\tau \geq 
	g(\nu,d) N^{\frac{d}{d+5}} \triangleq T_{max}.
	\end{equation}
	
	Now since $\varepsilon_N = \left( \frac{2^{5d+18}}{\nu} \cdot 630 (2d+11)(d+2)^4 \frac{1}{N} \right)^{\frac{1}{d+5}}$, we get that
	\begin{equation}
	\tau^* \geq 
	\frac{1}{d+5}
	\log_2
	\left(
	\left(
	\frac{\nu}{148\cdot 20}
	\right)^{d+5}
	\frac{\nu}{630 (2d+11)(d+2)^4}
	\frac{1}{2^{5d+18}} \cdot  N
	\right).
	\end{equation}
	Also, 
	\begin{equation}
	T_{MP(1)} + \sum_{\tau=1}^{\tau^*-1} D_\tau
	\geq
	\sum_{\tau=1}^{\tau^*-1} D_\tau
	=
	\sum_{\tau=1}^{\tau^*-1}
	\frac{\nu}{2} M_\tau,
	\end{equation}
where we used the fact that $D_\tau = \frac{\nu}{2} M_\tau$. Recall $M_\tau \triangleq C_M \frac{1}{\varepsilon_\tau^{d+2}} \ln (\frac{2}{\varepsilon_\tau})$, where $C_M = \frac{2^{\max(3.5d,8)}}{\nu} (3d+1)$, and for $\tau \leq \tau^*$ we have $\varepsilon_\tau = C/2^\tau$, where $C = \nu/(148 \cdot 20)$. Then
	we get
	\begin{align}
	T_{MP(1)} + \sum_{\tau=1}^{\tau^*-1} D_\tau
	& \geq
	\frac{\nu}{2}
	C_M \left(1/C \right)^{d+2}
	\sum_{\tau=1}^{\tau^*-1}
	2^{\tau(d+2)} \nonumber \\
	& \geq
	\frac{\nu}{4}
	C_M \left(1/C\right)^{d+2}
	2^{\tau^*(d+2)} \nonumber \\
	& \geq
	\underbrace{
		\frac{\nu}{4} C_M \left(1/C\right)^{d+2}
		\left(
		\frac{\nu}{630 (2d+11)(d+2)^4}
		\frac{1}{2^{5d+18}}
		\right)^{\frac{d+2}{d+5}}
	}_{g(\nu,d)}
	\cdot N^{\frac{d+2}{d+5}} \triangleq T_{max},
	\end{align}
	as wished. Hence, between $T_{min}$ and $T_{max}$ the target $\varepsilon_\tau$ for the epochs is indeed halving for each subsequent epoch. Let $\tau(T)$ be the epoch of time $T$. Then, by \cref{lmm:ColdStart}, for $T \in [T_{min},T_{max}]$, where $T_{min} = T_{MP}+T_{min,1}$, the expected regret satisfies
	\begin{equation}
	\RR(T) -T_{MP}  \leq 
	\frac{148}{\nu}
	\sum_{\tau=1}^{\tau(T)}\varepsilon_\tau D_\tau,
	\end{equation}
	which we can further bound as
	\begin{align}
	\RR(T) -T_{MP}
	& \leq
	\frac{C_M}{2} \log_2 \bigg( \frac{2^{\tau(T)}}{2 C} \bigg) 
	\sum_{\tau=1}^{\tau(T)} 2^{\tau (d+1)} \nonumber \\
	& \leq
	\frac{C_M}{2} \log_2 \bigg( \frac{2^{\tau(T)}}{2 C} \bigg)
	2^{(\tau(T)+1)(d+1)}.
	\end{align}
	Now, since for $T>T_{min}$ the epoch $\tau(T)$ is at most $1+\frac{1}{d+2}\log_2 \left(\frac{T-T_{MP}}{C_M} \frac{1}{\log(2/C)}\right)$, we get that
	\begin{align}
	\RR(T) & \leq 
	T_{MP} 
	+
	\frac{C_M}{2} 
	\log_2
	\left(
	\frac{1}{C (d+2)}\frac{1}{\log(2/C)}
	\frac{T-T_{MP}}{C_M}
	\right)
	2^{2(d+1)}
	2^{(\tau(T)+1)(d+1)}
	\end{align}
	$$
	\leq
	T_{MP } 
	+
	\underbrace{
		\frac{C_M}{2} 
		\log_2
		\left(
		\frac{1}{C (d+2)}\frac{1}{\log(2/C)}
		\frac{1}{C_M}
		\right)
		2^{4(d+1)}
	}_{\triangleq C'}
	\log_2
	\left(T-T_{MP}\right)
	2^{\frac{d+1}{d+2}\log_2 \left(\frac{T-T_{MP}}{C_M}\frac{1}{\log(2/C)}\right)}
	$$
	$$
	\leq
	T_{MP} 
	+
	\underbrace{
		C' \left(\frac{1}{C_M} \frac{1}{\log(2/C)}\right)^{\frac{d+1}{d+2}}
	}_{\triangleq \alpha(\nu,d)}
	\big( T-T_{MP}\big)^{\frac{d+1}{d+2}}
	\log_2
	\left(T-T_{MP}\right),
	$$
	as we wished, which completes the proof of the sublinear regret regime.
	
	The case $T>T_{max}$ now follows. Recall that by \cref{lmm:ColdStart} we get
	\begin{equation}
	\RR(T) \leq 
	T_{MP} 
	+
	\frac{148}{\nu}
	\sum_{\tau=1}^{\tau(T)}\varepsilon_\tau D_\tau,
	\end{equation}
	which we can in turn split between before $T_{max}$ and after $T_{max}$ as
	\begin{equation}
	\RR(T) \leq 
	T_{MP} 
	+
	\frac{148}{\nu}
	\sum_{\tau=1}^{\tau^*-1} \varepsilon_\tau D_\tau
	+
	\frac{148}{\nu}
	\sum_{\tau=\tau^*}^{\tau(T)} \varepsilon_\tau D_\tau
	\end{equation}
	\begin{equation}
	\leq
	\underbrace{
		T_{MP} 
		+
		\alpha(\nu,d)
		T_{max}^{\frac{d+1}{d+2}}
		\log_2
		\left(T_{max}\right)
	}_{\triangleq \beta}
	+
	\varepsilon_N (T-T_{max}),
	\end{equation}
	as claimed, and where the last inequality is due to the sublinear regime proved above.
\end{proof}

We are now ready to bound the cold-start time of $\proc{item-item-cf}$. Recall that cold-start time of a recommendation algorithm $\mathcal{A}$ is defined as the least $T+\Gamma$ such that for all $\Delta>\Gamma$ we have $\EE\left[ \RR^{(\mathcal{A})}(T+\Delta)-\RR^{(\mathcal{A})}(T) \right] \leq 0.1 \Delta$.

\corolColdStart*
\begin{proof}	
	First recall the usual definitions: $D_\tau = \frac{\nu}{2} M_\tau$, $T_\tau = T_{MP} + \sum_{\tau'<\tau} D_{\tau'}$, and $T_{MP} = f(\nu,d)/N$, where $f(\nu,d)$ is the number of recommendations required for the initial $\proc {make-partition}$ call (as stated in \cref{lmm:partitionConditions}), and $T_{min,\tau} = \frac{12}{\varepsilon_\tau} \ln \left(\frac{1}{\varepsilon_\tau}\right)$ (as stated in \cref{lmm:ColdStart}). We will show the bound in the definition of cold-start time with $T=T_{MP}$ and $\Gamma=T_{min,1}$, which implies $T_{\mathrm{cold-start}} = T_{MP} + T_{min,1}$.
	
To complete the proof, we shall establish the following two properties:
	\begin{itemize}
		\item[$(i)$] For any $\Delta>0$, $\EE \left[ \RR (T_{\mathrm{MP}}+ T_{\mathrm{min},1}+\Delta) -\RR (T_{\mathrm{MP}}) \right] \leq 0.1 (T_{\mathrm{min},1}+\Delta)$, for $T_{\mathrm{MP}}+ T_{{min},1}+\Delta \leq T_2$. This condition says that the desired property holds for times involving the first epoch, and
		\item[$(ii)$] $\EE \left[ \RR (T_\tau+\Delta) -\RR (T_\tau) \right] \leq 0.05 (\Delta+D_{\tau-1})$, for $\Delta \leq D_\tau$ and $\tau \geq 2$.
	\end{itemize}
Before we show how the above two properties imply the desired result, we note that $(i)$ follows directly from \cref{lmm:ColdStart}, and 
$(ii)$ will be proved at the end.
	
Now let complete the proof using (i) and (ii). To that end, consider a time of the form $T_{\mathrm{cold-start}} + \Delta = T_{MP}+T_{min,1}+\Delta$, for any $\Delta > 0$. Let $\tau^* \geq 1$ be the epoch to which $T_{MP}+T_{min,1}+\Delta$ belongs, i.e. 
\begin{align*}
T_{\tau^*} & < T_{MP}+T_{min,1}+\Delta \leq T_{\tau^*+1}.
\end{align*}
Define $t = T_{MP}+T_{min,1}+\Delta -T_{\tau^*} > 0$. We shall argue $\tau^* = 1$ and $\tau^* > 1$ separately. For $\tau^* =1$, (i) implies the desired result. For $\tau^* > 1$, we use (ii) to argue it as follows:
\begin{align}
& \EE\left[\RR\left(T_{\mathrm{MP}}+ T_{\mathrm{min},1}+\Delta \right)-\RR \left(T_{\mathrm{MP}}\right)\right]
 \underbrace{\leq}_{t = T_{\mathrm{cs}}+\Delta-T_{\tau^*}} \underbrace{\EE\left[\RR\left(T_{2}\right)-\RR \left(T_{\mathrm{MP}}\right)\right]}_{\leq 0.1 \cdot D_1 \text{ by $(i)$}}. \nonumber\\
& \qquad + \mathbf{1}_{\tau^*\geq 3} \bigg(\sum_{\tau=2}^{\tau^*-1}\underbrace{\EE\left[\RR\left(T_{\tau}\right)-\RR \left(T_{\tau-1}\right)\right]}_{\leq 0.05 (D_\tau+D_{\tau-1}) \text{ by $(ii)$}}\bigg) 
+\overbrace{\EE\left[\RR\left(T_{\tau^*}+t\right)-\RR \left(T_{\tau^*}\right)\right]}^{\leq 0.05 (t+D_{\tau-1}) \text{ by $(ii)$}} \nonumber \\
	& \qquad \leq 0.05 \left(t + 2 \sum_{\tau=1}^{\tau^*-1} D_\tau\right)  \nonumber \\
	& \qquad \leq 0.1 \cdot (\Delta + T_{\mathrm{min},1}).
\end{align}
	This establishes the desired result that $T_{\mathrm{cold-start}} = T_{MP}+T_{min,1}= f(\nu,d)/N+\mathcal{\widetilde O}(1/\nu)$.
	
	\emph{Proof of $(ii)$:} Now we argue the remaining property (ii). 
	\cref{lmm:ColdStart} tells us that for $\Delta \in (T_{min,\tau},D_\tau)$ we have that $\EE
	\left[
	\RR\left(T_\tau + \Delta \right)
	-
	\RR \left(T_{\tau} \right)
	\right] \leq \frac{148}{\nu} \varepsilon_\tau \Delta$, i.e. 
	$$\EE
	\left[
	\RR\left(T_\tau + T_{min,\tau} \right)
	-
	\RR \left(T_{\tau} \right)
	\right] \leq \frac{148}{\nu} \varepsilon_\tau T_{min,\tau}.$$
	Thus,  for $\Delta < D_\tau$, we have that
	\begin{equation}
	\EE
	\left[
	\RR\left(T_\tau + \Delta \right)
	-
	\RR \left(T_{\tau} \right)
	\right]
	\leq
	\frac{148}{\nu} \varepsilon_\tau T_{min,\tau}
	+
	\underbrace{
		\frac{148}{\nu} \varepsilon_\tau \Delta
	}_{\leq 0.05 \text{ for } \tau \geq 2}.
	\end{equation}
	In above we used the fact that for $\Delta < T_{{min}, 1}$, $\RR\left(T_\tau + \Delta \right) \leq \RR\left(T_\tau + T_{{min}, 1} \right)$. 
	Using the fact that $T_{min,\tau} \triangleq \frac{12}{\varepsilon_\tau} \ln \left(\frac{1}{\varepsilon_\tau}\right) \leq 0.05 \Big( \frac{300}{\varepsilon_\tau} \ln (\frac{1}{\varepsilon_\tau})\Big) = 0.05 D_{\tau-1}$, we conclude
	\begin{equation}
	\EE
	\left[
	\RR\left(T_\tau + \Delta \right)
	-
	\RR \left(T_{\tau} \right)
	\right]
	\leq
	0.05
	\left(\Delta+D_{\tau-1}\right).
	\end{equation}
	\end{proof}

\thmLowerBoundLinearRegime*
\begin{proof}
	Let $\{i_1, ..., i_{k_T}\}$ be the set of distinct items that have been recommended up to time $TN$. Then we have 
	\begin{align*}
	\EE \left[\mathcal{R} (T) \right]
	=
	\frac{1}{N} \EE \left[\sum_{t=1}^{TN} \frac{1}{2} (1-L_{U_t,I_t}) \right]
	&=
	\frac{1}{N} \EE \left[ \sum_{k=1}^{k_T} \sum_{t=1}^{TN} \frac{1}{2} \mathbf{1}_{I_t = i_k } (1-L_{U_t,i_k}) \right]
	\\&\geq 
	\frac{1}{N} \EE \left[ \sum_{k=1}^{k_T} \frac{1}{2}  (1-L_{U_{T_k},i_k}) \right],
	\end{align*}
	where $T_k$ is the first time in which the item $i_k$ is recommended to any user. Now note that for each $k$ by $(A_2)$ we have that $\EE \left[\frac{1}{2}  (1-L_{U_{T_k},i_k}) \right] \geq 1-2\nu$, since when we have no prior information about $i_k$ the best we can do is to recommend it to the user that likes the largest fraction of items. Hence we get
	\begin{equation}
	\EE \left[\mathcal{R} (T) \right]
	\geq
	\frac{1-2\nu}{N} k_t.
	\end{equation}
	Since each item can be recommended to each user at most once, we see that by the $TN^{th}$ recommendation at least $T$ different items must have been recommended (that is, $k_t \geq T$). We can then conclude that
	\begin{equation}
	\EE \left[\mathcal{R} (T) \right]
	\geq
	\underbrace{
		\frac{1-2\nu}{N} 
	}_{C(\nu,d)},
	\end{equation}
	as we wished.
\end{proof}

\section{More on Doubling Dimension}
\label{sec:DDAppendix}

In this section we provide examples of spaces with low doubling dimension, give some useful properties, and describe experiments indicating that doubling dimension is often small in practice.

				\begin{restatable}{examp}{exampDDExample}
					\label{examp:DDExample}
					Consider an item space $\mu$ over $N$ users that assigns probability at least $w>0$ to $K$ distinct item types with separation at least $\sigma>0$. Then, since $\mu(\BB(x,\alpha)) \leq 1$, and $\mu(\BB(x,\alpha/2)) \geq w$, we have that 
					\begin{equation}
					d = \max_{x \in \{-1,+1\}^N} \sup_{r}
					\frac{\mu(\BB(x,r))}{\mu(\BB(x,r/2))}
					\leq \frac{\mu(\BB(x,\alpha))}{\mu(\BB(x,\alpha/2))} \leq \log_2 \left(\frac{1}{w}\right) =\log_2 1/w.
					\end{equation}
					Similarly, if we only know that there are at most $K$ equally likely item types we can bound the doubling dimension as
					\begin{equation}
					d = \max_{x \in \{-1,+1\}^N} \sup_{r}
					\frac{\mu(\BB(x,r))}{\mu(\BB(x,r/2))}
					\leq \frac{1}{1/K} = \log_2 K.
					\end{equation}
					\end{restatable}
					
					With the example above in mind, we would like to emphasize that doubling dimension assumptions are strictly more general than the style of assumptions made in \cite{bresler14} (finite $K$ with separation assumptions) because $(a)$ doubling measure require no separation assumptions (that is, two item types $x$ and $y$ that are arbitrarily close to each other can have positive mass) and $(b)$ the number of types of positive mass is not bounded by a finite $K$ anymore, but instead can grow with the number of users.

					\begin{restatable}{examp}{exampDDRandom}
						\label{examp:DDRandom}
						Consider an item space $\mu$ such that it assigns probability $1/K$ to $K$ item types randomly uniformly drawn from $\{-1,+1\}^N$. Then, for each two item types $i,j$ we have that
						\begin{equation}
						\Prob
						\left(
						\bigcup_{ij} \{\gamma_{ij} \notin [.4,.6] \}
						\right)
						\leq 
						\sum_{ij}
						\Prob
						\bigg( 
						\gamma_{ij} \notin [.4,.6]
						\bigg)
						\leq 
						{K \choose 2}
						\exp(-\Theta(N)),
						\end{equation}
						where the first inequality is due to a union bound, and the second to a Chernoff bound. Hence, with high probability we get that
						$d \geq \frac{\mu(\BB(x,.7))}{\mu(\BB(x,.35))} 
						\underbrace{=}_{w.h.p.} 
						\log_2 \frac{1}{1/K} = \log_2 K.$
						By \cref{examp:DDExample} we also have that $d \leq \log_2(K)$, and hence we can conclude that with high probability we have that $d = \log_2(K)$.
						\end{restatable}

\begin{restatable}{prop}{clmDoublingDimensionClaim2}
	\label{clm:DoublingDimensionClaim2}
	Let $\mu$ be an item space for $N$ users with doubling dimension $d$. Then for any item type $x \in \{-1,1\}^N$ with $\mu(x) > 0$ we have
	\begin{equation}
	\mu \big(\BB\left( x,r \right) \big) \geq r^d.
	\end{equation}
\end{restatable}

\begin{restatable}{prop}{clmDoublingDimensionClaims}
	\label{clm:DoublingDimensionClaims}
	Let $\mu$ be an item space for $N$ users with doubling dimension $d$, let $\CC$ be an $\varepsilon$-net for $\mu$, let $j$ be an arbitrary item, let $c_j \in \CC$ be such that $\gamma_{j,c_j} < \varepsilon$, and let $m_{c_j} \triangleq \mu \big( \BB (c_j,\varepsilon)\big)$. Then, for each $r \in [\varepsilon/2,1/2]$, there are at most $m_{c_j} \left(\frac{4}{\varepsilon}\right)^d \left(\frac{4r+5\varepsilon}{4 \varepsilon}\right)^d$ items in $\CC$ within radius $r$ of $j$.
\end{restatable}

\begin{proof}
	By the doubling dimension of $\mu$ we get
	\begin{equation*}
	\label{eqDoublingTrick}
	\tag{$\star$}
	\mu \left( \BB \left( c_j, r+\frac{5}{4} \varepsilon \right) \right)
	\leq
	\mu \left( \BB \left(c_j, \varepsilon \right) \right)
	\cdot
	\left( \frac{r+5 \varepsilon /4}{\varepsilon}\right)^d
	=
	m_{c_j}
	\cdot
	\left( \frac{r+5 \varepsilon /4}{\varepsilon} \right)^d
	.
	\end{equation*}
	
	We will now use this bound on $\mu \left( \BB \left(c_j, r+\frac{5}{4} \varepsilon\right)\right)$ to show that we could pack at most $m_{c_j} \left(\frac{4}{\varepsilon}\right)^d \left(\frac{4r+5\varepsilon}{4 \varepsilon}\right)^d$ items from $\CC$ within $r$ of $j$.
	
	Since $\CC$ is an $\varepsilon$-net, each two items $i,j \in \CC$ are at least $\varepsilon/2$ apart, and hence the balls of radius $\varepsilon/4$ around each $i \in \CC$ are disjoint. Say that there are $K \triangleq |\CC_j |$ items  $\CC$ within distance $r$ of $j$, where $\CC_j \triangleq  \{ c \in \CC \mid \gamma_{c,j} \leq r \}$. Then we get 
	\begin{equation}
	K \cdot \left(\frac{\varepsilon}{4}\right)^d
	\leq
	\sum_{c \in \CC_j} \mu \left(\BB(c,\varepsilon/4)\right)
	=
	\mu \big( \bigcup_{c \in \CC_j} \BB \left( c,\varepsilon/4 \right) \big)
	\leq
	\mu \big(\BB(c_j , r+5\varepsilon/4) \big),
	\end{equation}
	where the first inequality is due to \cref{clm:DoublingDimensionClaim2}  and the last inequality is due to $\cup_{c \in \CC_j} B(c,\varepsilon/4)  \subset B(c_j , r+5\varepsilon/4)$. Using the bound from \cref{eqDoublingTrick} we arrive at
	\begin{equation}
	K
	\leq 
	\left(\frac{4}{\varepsilon}\right)^d
	\cdot
	m_{c_j}
	\cdot
	\left( \frac{r+5 \varepsilon /4}{\varepsilon} \right)^d.
	\end{equation}
	as we wished.
\end{proof}

	Finally, we would like to note that doubling dimension is not only a ``proof technique'': it can be estimated from data and tends to be small in practice. To illustrate this point, we calculate the doubling dimension on the Jester Jokes Dataset\footnote{\cite{goldberg2001eigentaste}, and data available on \url{http://goldberg.berkeley.edu/jester-data/}} and for the MovieLens $1$M Dataset\footnote{ \cite{riedl1998movielens}, and data available on \url{http://grouplens.org/datasets/movielens/}}. For the MovieLens dataset we consider the only movies that have been rated by at least $750$ users (to ensure some density). 
	
	The Jester dataset contains ratings of one hundred jokes by over seventy thousand users. The dataset is fairly dense (as the average number of ratings per user is over fifty), which makes it a great dataset for calculating the doubling dimension. For the MovieLens $1$M Dataset we consider the only movies that have been rated by at least $750$ users (to ensure some density).
	
	The Jester ratings are in $[-10,10]$, with an average of $2$, so we make ratings greater than $2$ a $R_{u,i} = +1$, and ratings at most $2$ a $R_{u,i}=-1$. For the MovieLens $1$M Dataset we make ratings $1,2,3$ into $-1$, and $4,5$ into $+1$. We then estimate the doubling dimension as follows:
	\begin{itemize}
		\item For each pair of items $(i,j)$, we calculate $\hat d_{i,j,\Delta}$ as fraction of users that agree on them, where the $\Delta$ subscript is put to denote our assumption that each entry has a noise probability of $\Delta$ (that is, $\Prob (R_{u,i} \ne L_{u,i}) = \Delta$), where $R$ is the empirical ratings matrix and $L$ is the true, noiseless, ratings matrix.
		\item Assuming that each entry has a noise probability of $\Delta = 0.20$, we estimate the true distance $d_{i,j}$ as the solution to $\hat d_{i,j,\Delta} = (1-d_{ij})(2\Delta(1-\Delta))+d_{i,j}(\Delta^2+(1-\Delta)^2)$.
		\item For each item $i$ and $r$ in $\{0,\frac{1}{N},...,\frac{N-1}{N},1\}$, let $N_{i,r}$ be the number of items such that $d_{i,j} \leq r$.
		\item For each item $i$ let $d_i$ be the least such that $N_{i,2r}/N_{i,r} \leq 2^{d_i}$ for each $r$ in $\{0,\frac{1}{N},...,\frac{1}{2}\}$.
		\item The \cref{fig:EmpiricalDoublingDimensionsJester,fig:EmpiricalDoublingDimensionsMovieLens} show the histogram of the $\{d_i\}$.
	\end{itemize}
	
	\begin{figure}[h]
		\centering
		\begin{minipage}[]{0.42\linewidth}
			\includegraphics[width=1.2\textwidth]{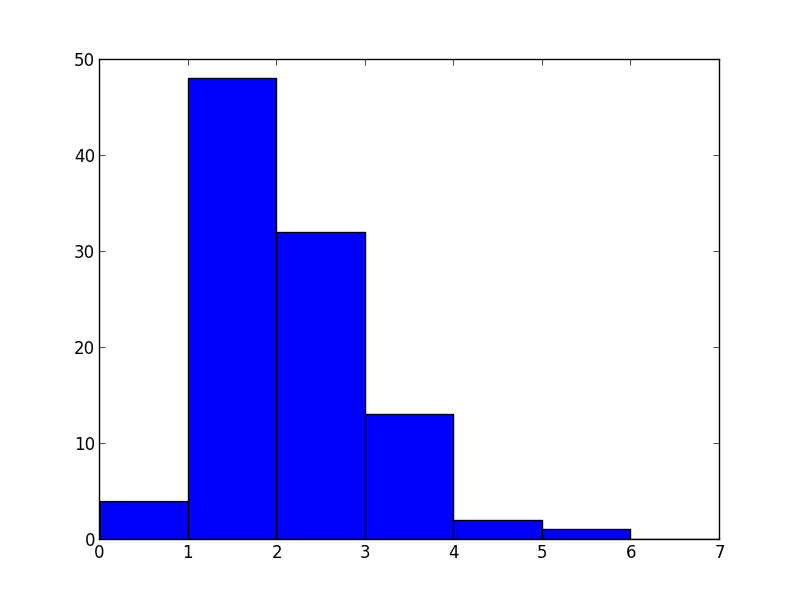}
			\caption{Jester Doubling Dimensions}
			\label{fig:EmpiricalDoublingDimensionsJester}
		\end{minipage}
		\quad
		\begin{minipage}[]{0.42\linewidth}
			\includegraphics[width=1.2\textwidth]{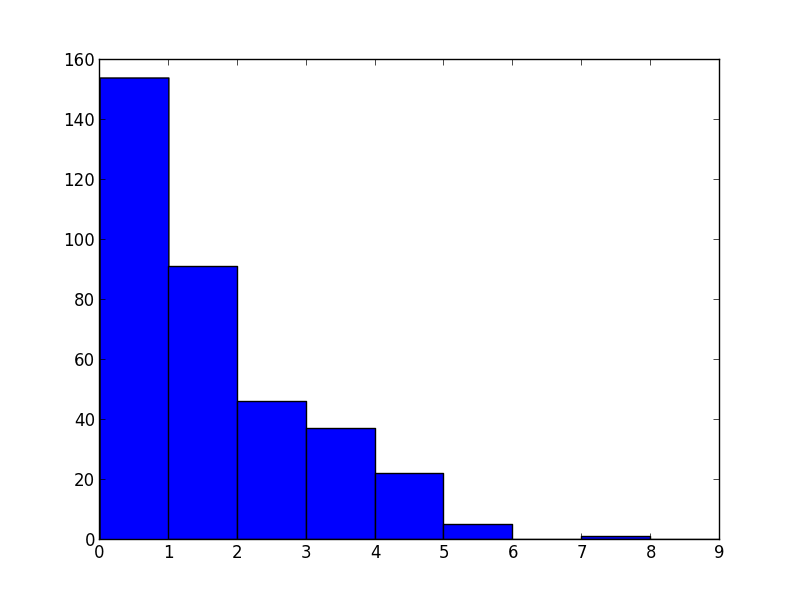}
			\caption{MovieLens Doubling Dimensions}
			\label{fig:EmpiricalDoublingDimensionsMovieLens}
		\end{minipage}
	\end{figure}

\newpage
\section{Chernoff Bound}
The following is a standard version of Chernoff Bound \citep{mcdiarmid1998concentration} that we use throughout the paper.

\begin{restatable}[Chernoff Bound]{thm}{thm:Chernoff}
	\label{thm:ChernoffBound}
	Let $X_1,\cdots,X_n$ be independent random variables that take value in $[0,1]$. Let $X = \sum_{i=1}^n X_i$, and let $\bar{X} = \sum_{i=1}^n \EE X_i$. Then, for any $\varepsilon \geq 0$,
	\begin{align*}
	& \Prob
	\left(
	X \geq \left(1+\varepsilon\right) \bar{X}
	\right)
	\leq
	\exp 
	\left(
	- \frac{\varepsilon^2}{2+\varepsilon} \bar{X}
	\right) \text{, and} \\
	& \Prob
	\left(
	X \leq \left(1-\varepsilon\right) \bar{X}
	\right)
	\leq
	\exp 
	\left(
	- \frac{\varepsilon^2}{2} \bar{X}
	\right).
	\end{align*}
\end{restatable}

\end{document}